\documentclass{article} 
\usepackage{iclr2027_conference,times}

\usepackage{amsmath,amsfonts,bm}

\def\eqref#1{Equation~(\ref{#1})}

\def\1{\bm{1}}

\DeclareMathAlphabet{\mathsfit}{\encodingdefault}{\sfdefault}{m}{sl}
\SetMathAlphabet{\mathsfit}{bold}{\encodingdefault}{\sfdefault}{bx}{n}

\DeclareMathOperator*{\argmin}{arg\,min}

\iclrfinalcopy
\usepackage{hyperref}
\usepackage{url}
\usepackage{booktabs}
\usepackage{tabularx}
\usepackage[table]{xcolor}
\usepackage{multirow}
\usepackage{graphicx}
\usepackage{wrapfig}
\usepackage[inline]{enumitem}

\usepackage{amsmath}
\usepackage{amsthm}
\newtheorem{theorem}{Theorem}
\newtheorem{lemma}{Lemma}
\newtheorem{prop}{Proposition}
\newtheorem{corollary}{Corollary}
\newtheorem{remark}{Remark}

\usepackage{minitoc}

\title{Graph-Spectral Flow Matching for \\Multivariate Time Series Anomaly Detection}

\author{
Zepeng Zhang$^{1}$ \quad
Jhony H. Giraldo$^{2}$ \quad
Wenbin Wang$^{3}$ \quad
Olga Fink$^{1}$ \\[0.6em]
$^{1}$ IMOS, EPFL \quad
$^{2}$ LTCI, Télécom Paris, IP Paris \quad
$^{3}$ Independent Researcher
}

\begin{document}

\maketitle
\lhead{Preprint}
\begin{abstract}
Multivariate time series anomaly detection typically relies on evaluating discrepancies between observations and outputs produced by models trained on normal data. An alternative perspective is to characterize the distribution of normal data through the generative dynamics, \textit{i.e.}, the velocity field, of flow matching models. However, standard flow matching typically adopts linear probability paths that overlook dependencies among variables, leading to a misalignment with the structured data distribution. To address this issue, we propose GRASP, a flow matching framework with a graph-spectral path for multivariate time series anomaly detection. GRASP incorporates graph structure into the probability path by minimizing a fixed-endpoint action that combines kinetic energy with graph Dirichlet energy. This formulation yields a closed-form path based on graph-frequency-dependent hyperbolic interpolation. A velocity predictor trained on normal data then detects anomalies using weighted velocity discrepancies aggregated across source samples, flow times, and graph frequencies. Theoretically, we establish that GRASP is invariant to the choice of Laplacian eigenbasis and decompose its expected oracle anomaly score into bounded endpoint uncertainty and graph-frequency-weighted Fisher discrepancy. Experiments on four benchmarks demonstrate the superior anomaly detection performance of GRASP and validate the effectiveness of its graph-spectral path and weighting mechanism.

\end{abstract}

\section{Introduction}
Industrial systems are increasingly monitored by sensor networks that generate large volumes of multivariate time series (MTS) data.
Detecting anomalies in these data is critical for preventing failures, reducing downtime, and ensuring reliable operation \citep{wu2021graph,deng2021graph,fink2026physics}.
However, fault labels are typically scarce and incomplete, as many fault types occur rarely or are not observed during data collection, motivating unsupervised approaches to MTS anomaly detection \citep{zhang2019deep,audibert2020usad,belay2023unsupervised}.
Existing methods typically learn normal patterns and identify anomalies through prediction errors, reconstruction errors, or representation discrepancies \citep{jin2024survey,chen2024graph,ho2025graph}.
These approaches are effective when abnormal behavior produces clear discrepancies between observations and model outputs.
However, in practice, anomalous patterns may remain partially predictable or reconstructible, resulting in weak anomaly signals \citep{liu2025gcad,cho2025structured,zhang2026gslad}.
This motivates anomaly detection methods that go beyond measuring only endpoint-based prediction or reconstruction discrepancies.

Flow matching offers an alternative paradigm by characterizing observations through generative dynamics rather than solely through endpoint outputs, which has been shown to be effective in image anomaly detection \citep{chen2026flow}.
During training time, a velocity predictor learns to match conditional target velocities along prescribed probability paths connecting source samples to normal observations.
During inference time, anomalies are detected by measuring discrepancies between the predicted velocities and conditional target velocities along the entire probability path.
The probability path determines both the states at which an observation is evaluated and the conditional target velocity against which it is compared \citep{lipman2023flow,liu2023flow,albergo2024stochastic}.
Therefore, the design of the probability path is crucial for flow-matching-based anomaly detection.

MTS variables often exhibit structured dependencies that can be represented by a sensor graph, providing an informative inductive bias for modeling cross-variable interactions \citep{bronstein2021geometric, jin2024survey}.
The graph Laplacian decomposes multivariate signals into graph-frequency modes that describe different patterns of variation across connected nodes \citep{shuman2013emerging}.
Recent work further shows that anomalous behavior can induce heterogeneous energy shifts across graph frequencies, suggesting that different frequency modes carry distinct anomaly-relevant information \citep{liu2026modeling}.
However, existing flow matching models typically adopt a standard linear conditional path, which interpolates all directions identically without accounting for the underlying graph structure \citep{kollovieh2025flow,albergo2024stochastic}.
For structured data, this structural mismatch may produce intermediate states and target velocities that are inconsistent with the topology and geometry of the monitored system \citep{rozada2026graph,fang2026escaping,zhang2026spatiotemporal}.
This limitation is more pronounced for MTS anomaly detection because the anomaly score depends not only on the observed endpoint, but also on the velocity discrepancies evaluated along the entire path.
These observations motivate the design of a graph-informed probability path for flow matching.

In this work, we introduce flow matching with a GRAph-SPectral Path (GRASP), which explicitly incorporates structural information into the conditional probability path.
When graph edges encode similarity between normal sensor signals, graph Dirichlet energy provides a natural measure for regularizing variation across connected nodes. We therefore formulate path construction as a fixed-endpoint variational problem that balances kinetic energy and graph Dirichlet energy. 
The resulting closed-form solution assigns a graph-frequency-dependent hyperbolic interpolation schedule to each frequency mode. Specifically, the zero-frequency modes recover standard linear interpolation, whereas higher-frequency modes undergo progressively stronger contraction at intermediate flow times. 
In this way, GRASP preserves the prescribed source and data endpoints while imposing a graph-dependent smoothness prior along the probability path. 
The velocity predictor trained on normal data then detects anomalies by measuring velocity discrepancies along this structured path.

The resulting path induces different conditional uncertainty of the
intermediate state and velocity-residual scales across flow times and graph frequencies. To account for these differences, GRASP uses flow-time- and graph-frequency-dependent weights when aggregating velocity discrepancies into an anomaly score. Theoretically, we establish that the node-domain path, target velocity, and resulting anomaly score are invariant to the choice of orthonormal basis within repeated Laplacian eigenspaces. Under the population-optimal velocity predictor, we further show that the expected anomaly score decomposes into an endpoint-uncertainty term and a graph-frequency-weighted Fisher discrepancy. Experiments on four MTS benchmarks demonstrate superior anomaly detection performance of GRASP and the effectiveness of the graph-spectral path and the weighting schedule.

Our contributions are summarized below:
\begin{itemize}[leftmargin=0.5cm]
    \itemsep0em
    \item We derive a closed-form graph-spectral probability path from a fixed-endpoint variational problem combining kinetic energy and graph Dirichlet energy, yielding frequency-dependent hyperbolic interpolation. Based on the graph-spectral path, we introduce an anomaly score that aggregates weighted velocity discrepancies across source samples, flow times, and graph frequencies.

    \item We establish that GRASP is invariant to different Laplacian eigenbases, and decompose the expected oracle anomaly score into bounded endpoint uncertainty and weighted Fisher discrepancy.

    \item We conduct experiments on four MTS benchmarks, demonstrating that GRASP achieves competitive or superior performance compared with the baselines. 
    The ablation and sensitivity studies validate the effectiveness of the graph-spectral path design and anomaly scoring mechanism.
\end{itemize}
Further discussion of related work is provided in Appendix \ref{app:related_work}.
\section{Preliminaries}
\textbf{Notations and Problem Definition.}
We use calligraphic letters, such as $\mathcal{X}$, to represent sets, uppercase bold letters, such as $\mathbf{X}$, to represent matrices, lowercase bold letters, such as $\mathbf{x}$, to represent vectors, and lowercase letters, such as $x$, to represent scalars.
A complete summary of the notation is provided in
Appendix~\ref{app:notation}.
We consider an MTS observation over $N$ variables. 
A time-series window is represented as $\mathbf{X}\in\mathbb{R}^{N\times R}$, where its $r$-th column $\mathbf{x}_r\in\mathbb{R}^{N}$ contains the observations of all $N$ variables at timestep $r$, and $R$ denotes the window length.
The relationships among the variables are captured by a graph $\mathcal{G}=(\mathcal{N},\mathcal{E})$, where $\mathcal{N}$ and $\mathcal{E}$ denote the sets of nodes and edges, respectively.
We denote by $\mathbf{L}\in\mathbb{R}^{N\times N}$ the symmetric normalized graph Laplacian matrix.
Since $\mathbf{L}$ is real, symmetric, and positive semidefinite, it admits an eigendecomposition $\mathbf{L}=\boldsymbol{\Psi}\mathbf{\Lambda}\boldsymbol{\Psi}^{\top},$ where $\boldsymbol{\Psi}=[\boldsymbol{\psi}_1,\ldots,\boldsymbol{\psi}_N]$ is an orthogonal matrix whose columns are the Laplacian eigenvectors and $\mathbf{\Lambda} =\operatorname{diag}(\lambda_1,\ldots,\lambda_N)$ contains the corresponding eigenvalues ordered as $0\leq\lambda_1\leq\cdots\leq\lambda_N\leq2$.

\textbf{MTS Anomaly Detection.}
During training, we observe only normal time-series windows sampled from an unknown normal data distribution $p$.
At test time, observations are drawn from a potentially anomalous distribution $q$.
Let $p_t$ and $q_t$ denote the corresponding path marginals at flow time $t$. 
Our goal is to learn the generative dynamics of normal data without anomaly samples and construct an anomaly score that quantifies the deviation of a test observation from the learned normal dynamics.

\textbf{Conditional Flow Matching.}
Flow matching learns a time-dependent velocity field that transports $\mathbf{X}_0\sim p_0$ to $\mathbf{X}_1\sim p_1$ along a prescribed probability path \citep{lipman2023flow, albergo2024stochastic}.
In this paper, the entries of \(\mathbf X_0\) are sampled independently from a standard Gaussian distribution, while $p_1$ corresponds to the normal data distribution $p$.
The conditional path between an endpoint pair $(\mathbf{X}_0,\mathbf{X}_1)$ is specified by an interpolation map and the corresponding conditional target velocity:
\begin{equation}
    \mathbf{X}_t
    =
    \phi_t(\mathbf{X}_0,\mathbf{X}_1),
    \qquad \mathbf{U}_t
    =
    \frac{\partial}{\partial t}
    \phi_t(\mathbf{X}_0,\mathbf{X}_1), \qquad t\in[0,1],
\end{equation}
which satisfies the boundary conditions $\phi_0(\mathbf{X}_0,\mathbf{X}_1)=\mathbf{X}_0$ and $\phi_1(\mathbf{X}_0,\mathbf{X}_1)=\mathbf{X}_1$.
Conditional flow matching trains a parameterized velocity field $\mathbf{V}_t(\mathbf{X};\boldsymbol{\theta})$ to approximate the conditional target velocity by minimizing
\begin{equation}
    \mathcal{L}_{\mathrm{CFM}}(\boldsymbol{\theta})
    =
    \mathbb{E}_{t\sim\mathcal{U}[0,1],\mathbf{X}_0,\mathbf{X}_1}
    \left[
        \left\|
        \mathbf{V}_t(\mathbf{X}_t;\boldsymbol{\theta})
        -
        \mathbf{U}_t
        \right\|_F^2
    \right],
\end{equation}
where $\mathcal{U}[0,1]$ is the uniform distribution between 0 and 1.
A commonly used probability path is the linear interpolation, where the intermediate state and the conditional target velocity are defined as
\begin{equation}
    \mathbf{X}_t
    =
    \phi_t(\mathbf{X}_0,\mathbf{X}_1)
    =
    (1-t)\mathbf{X}_0+t\mathbf{X}_1\quad \text{and}\quad
    \mathbf{U}_t
    =
    \mathbf{X}_1-\mathbf{X}_0.
\end{equation}
This path construction applies the same linear interpolation schedule to all directions in the ambient data space and therefore does not explicitly account for structural dependencies among variables.

\textbf{Graph Fourier Transform.}
The eigendecomposition of the graph aplacian matrix provides a graph Fourier basis for signals defined on the corresponding graph $\mathcal{G}$.
Given an MTS graph signal $\mathbf{X}\in\mathbb{R}^{N\times R}$, its graph Fourier transform and the inverse graph Fourier transform are defined as follows:
\begin{equation}
    \hat{\mathbf{X}}
    =
    \boldsymbol{\Psi}^{\top}\mathbf{X}
\quad\text{and}\quad
    \mathbf{X}
    =
    \boldsymbol{\Psi}\hat{\mathbf{X}}.
\end{equation}
The $k$-th row $\hat{\mathbf{x}}_k\in\mathbb{R}^{R}$ of $\hat{\mathbf{X}}$ contains the coefficients of the $k$-th graph-frequency mode across the $R$ timesteps.
The Laplacian eigenvalue $\lambda_k$ characterizes the graph-frequency associated with eigenvector $\boldsymbol{\psi}_k$.
Modes associated with smaller eigenvalues vary more smoothly over the graph, whereas modes associated with larger eigenvalues exhibit stronger variation across connected nodes. The graph Dirichlet energy admits a spectral decomposition that separates graph frequency modes:
\begin{equation}
    \operatorname{Tr}
    \left(
        \mathbf{X}^{\top}\mathbf{L}\mathbf{X}
    \right)
    =
    \operatorname{Tr}
    \left(
        \hat{\mathbf{X}}^{\top}
        \mathbf{\Lambda}
        \hat{\mathbf{X}}
    \right)
    =
    \sum_{k=1}^{N}
        \lambda_k
        \left\|
            \hat{\mathbf{x}}_k
        \right\|_2^2.
\end{equation}
This decomposition shows that the graph Dirichlet energy penalizes high-frequency components more strongly, thereby providing a natural measure of signal variation across connected nodes.
\begin{figure}[t]
    \centering
    \includegraphics[width=1.0\linewidth]{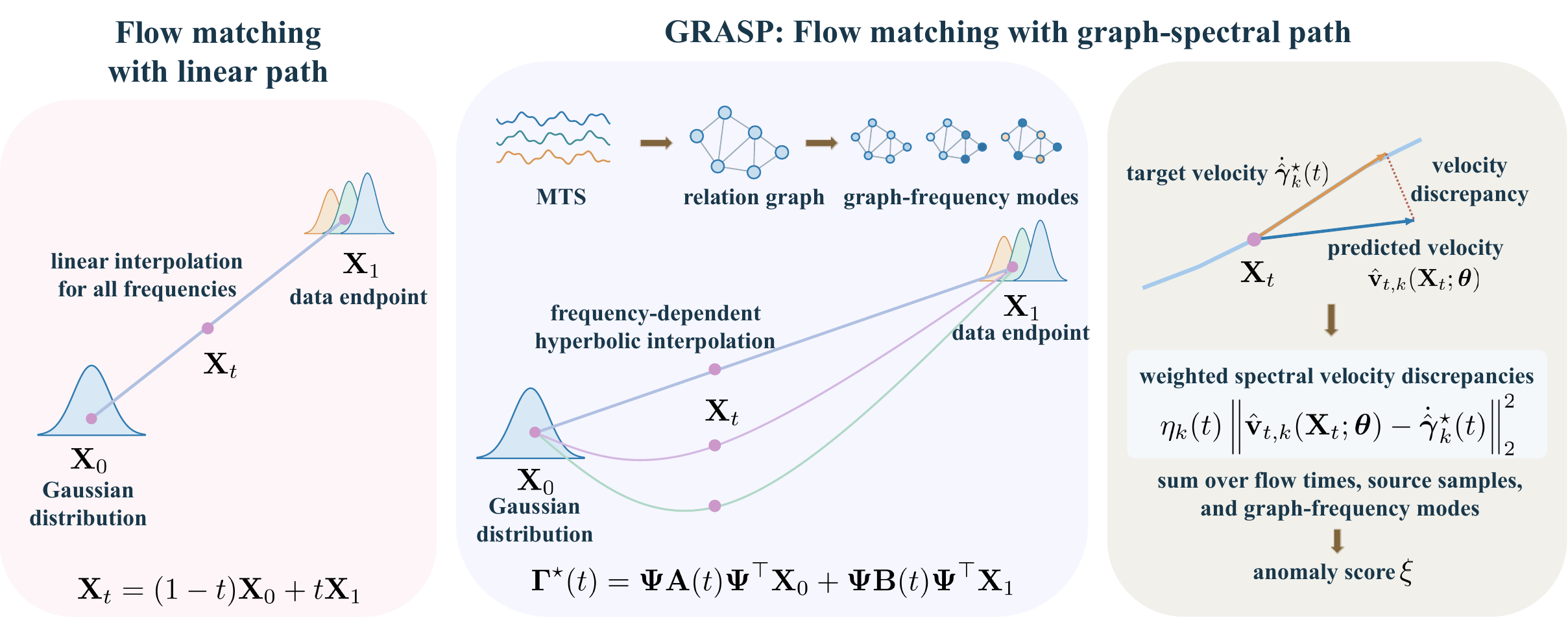}
    \caption{Overview of GRASP for MTS anomaly detection.
Left: flow matching with the same linear interpolation for all graph-frequency modes.
Middle: GRASP with the graph-spectral path that applies frequency-dependent hyperbolic interpolation. 
Right: anomaly scoring strategy of GRASP.}
    \label{fig:flowchart}
\end{figure}
\section{Flow Matching with Graph-Spectral Path}
The standard linear path used in flow matching applies the same interpolation schedule to all directions, implicitly imposing an isotropic transport geometry.
Ignoring the structural information in MTS data may produce intermediate states and target velocities that are poorly aligned with the underlying graph geometry.
In this section, we develop GRASP, which explicitly incorporates structural information into probability-path design through a graph-informed variational formulation.
A schematic comparison between GRASP and standard linear-path flow matching is given in Figure \ref{fig:flowchart}.
\subsection{Graph-Spectral Path Design}
The linear path widely adopted by flow matching models can be characterized as the minimizer of a variational problem with the kinetic energy Lagrangian function defined as follows \citep{du2026lagrangian}:
\begin{equation}
\boldsymbol{\Gamma}^\star=\argmin_{\boldsymbol{\Gamma}(0)=\mathbf{X}_0,\boldsymbol{\Gamma}(1)=\mathbf{X}_1}\int_0^1\mathcal{L}_\mathrm{lin}({\boldsymbol{\Gamma}},\dot{\boldsymbol{\Gamma}},t)dt \quad \text{with}\quad 
\mathcal{L}_\mathrm{lin}({\boldsymbol{\Gamma}},\dot{\boldsymbol{\Gamma}},t)=\frac{1}{2}\|\dot{\boldsymbol{\Gamma}}(t)\|^2_F,
\end{equation}
where $\boldsymbol{\Gamma}:[0,1]\rightarrow\mathbb{R}^{N\times R}$ is the path function. 
Because this objective penalizes only kinetic energy, it treats all directions in the ambient data space isotropically.
Although such a path is suitable for Euclidean data \citep{liu2023flow,lipman2023flow,albergo2024stochastic}, it may produce intermediate states that are poorly aligned with the structure of data supported on non-Euclidean domains \citep{fang2026escaping,wyrwal2026topological}.
In MTS anomaly detection, since velocity discrepancies are evaluated throughout the entire path, such structural misalignment can lead to less informative anomaly signals. 
Under the assumption that time-varying graph signals vary smoothly over the relation graph, connected nodes tend to exhibit coherent behavior under normal operation \citep{kalofolias2016learn,dong2016learning,giraldo2022reconstruction}. 
We therefore propose to construct a graph-informed probability path by considering a Lagrangian that combines kinetic energy with graph Dirichlet energy:
\begin{equation}
\mathcal{L}_\mathrm{graph}({\boldsymbol{\Gamma}},\dot{\boldsymbol{\Gamma}},t)=\frac{1}{2}\|\dot{\boldsymbol{\Gamma}}(t)\|^2+\frac{\tau}{2}\mathrm{Tr}\left(\boldsymbol{\Gamma}(t)^\top \mathbf{L}\boldsymbol{\Gamma}(t)\right),
\end{equation}
where the graph Dirichlet energy weight parameter $\tau\geq0$ controls the strength of the graph regularization.
The graph Dirichlet energy term penalizes abrupt variations across connected nodes in the graph at intermediate states.
As a result, higher graph-frequency modes undergo stronger contraction while preserving the endpoints.
The resulting frequency-dependent interpolation therefore treats different graph-frequency modes differently, respecting the phenomenon observed in \cite{liu2026modeling} that anomalies induce heterogeneous behavior across different graph frequency modes. 

To solve the resulting graph-informed variational problem, we first transform the path to the graph Fourier domain as $\hat{\boldsymbol{\Gamma}}(t)=\boldsymbol{\Psi}^\top\boldsymbol{\Gamma}(t)$.
Since $\boldsymbol{\Psi}$ is orthonormal and independent of $t$, we have $\|\dot{\boldsymbol{\Gamma}}(t)\|^2_F=\|\dot{\hat{\boldsymbol{\Gamma}}}(t)\|^2_F$. 
The graph-informed variational problem can therefore be written as follows:
\begin{equation}
\hat{\boldsymbol{\Gamma}}^\star=\argmin_{\hat{\boldsymbol{\Gamma}}(0)=\hat{\mathbf{X}}_0,\hat{\boldsymbol{\Gamma}}(1)=\hat{\mathbf{X}}_1}\int_0^1\frac{1}{2}\sum_{k}\|\dot{\hat{\boldsymbol{\gamma}}}_k(t)\|^2_2+\frac{\tau}{2}\sum_k {\lambda}_k\|\hat{\boldsymbol{\gamma}}_k(t)\|_2^2dt,
\label{eq:graph_variational_problem}
\end{equation}
where $\hat{\boldsymbol{\gamma}}_k(t)$ denote the $k$-th row of $\hat{\boldsymbol{\Gamma}}(t)$. 
Since both the objective and the boundary conditions are separable across different $k$, \textit{i.e.}, different graph frequencies, the problem in \eqref{eq:graph_variational_problem} decomposes into $N$ independent vector-valued variational problems:
\begin{equation}
\hat{\boldsymbol{\gamma}}_k^\star=\argmin_{\hat{\boldsymbol{\gamma}}_k(0)=\hat{\mathbf{x}}_{0,k},\hat{\boldsymbol{\gamma}}_k(1)=\hat{\mathbf{x}}_{1,k}}\int_0^1\frac{1}{2}\|\dot{\hat{\boldsymbol{\gamma}}}_k(t)\|^2_2+\frac{\tau}{2}{\lambda}_k\|\hat{\boldsymbol{\gamma}}_k(t)\|_2^2dt,
\label{eq:separated_variational_problem}
\end{equation}
where $\hat{\mathbf{x}}_{t,k}$ represents the $k$-th row of $\hat{\mathbf{X}}_{t}$.
Let $\omega_k=\sqrt{\tau\lambda_k}$, we have the following result.

\begin{theorem}[Graph-spectral path]\label{prop:graph_path}
    The variational problem in \eqref{eq:separated_variational_problem} admits a unique global minimizer given by
\begin{equation}
\hat{\boldsymbol{\gamma}}_k^\star(t)=\alpha_k(t)\hat{\mathbf{x}}_{0,k}+\beta_k(t)\hat{\mathbf{x}}_{1,k},  
\end{equation}
where 
\begin{equation}
    \alpha_k(t)=\frac{\sinh(\omega_k(1-t))}{\sinh(\omega_k)} \quad\text{and} \quad\beta_k(t)=\frac{\sinh(\omega_kt)}{\sinh(\omega_k)}.
\end{equation}
\end{theorem}

\begin{wrapfigure}{r}{0.7\columnwidth}
    \centering
    \includegraphics[
        width=\linewidth
    ]{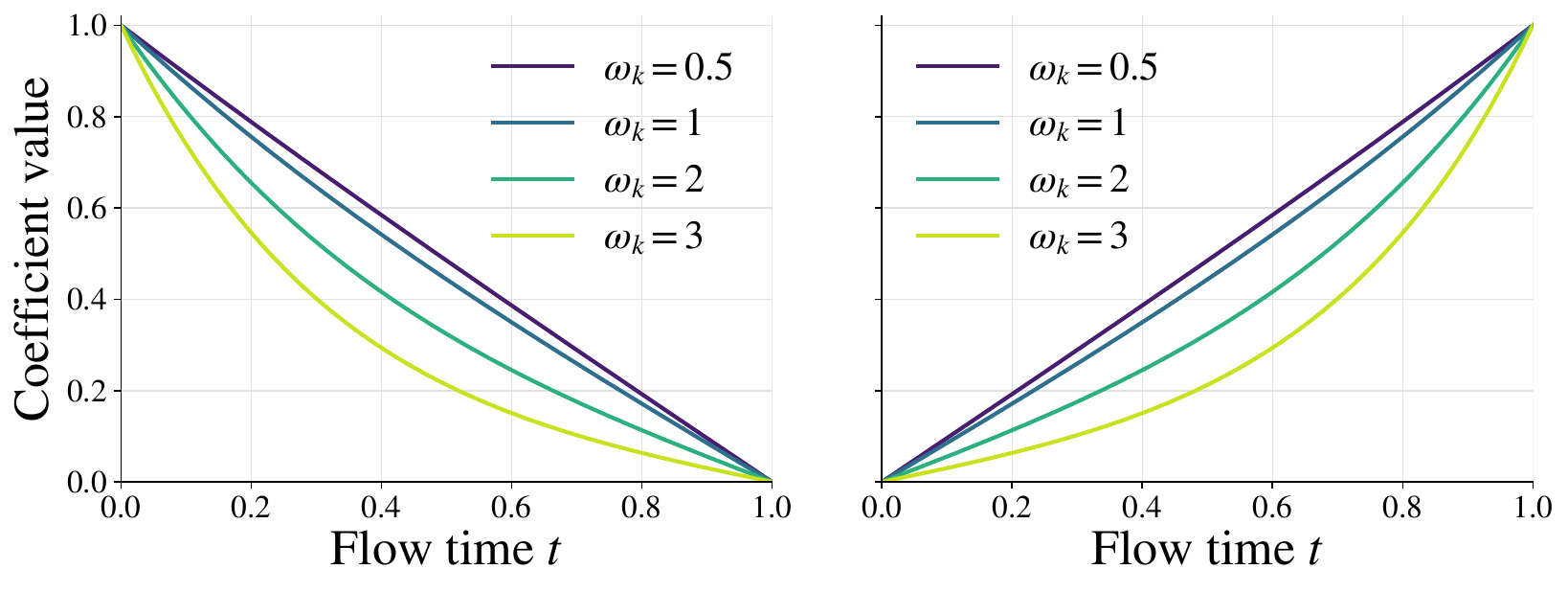}
    \caption{Coefficients $\alpha_k(t)$ (left) and $\beta_k(t)$ (right) with different $\omega_k$. 
    }
    \label{fig:graph_spectral_coefficients}
    \vspace{-8pt}
\end{wrapfigure}
The proof of Theorem \ref{prop:graph_path} is provided in Appendix \ref{proof:graph_path}.
Figure \ref{fig:graph_spectral_coefficients} visualizes the source coefficient $\alpha_k(t)$ and target coefficient $\beta_k(t)$ for different values of $\omega_k$.
The source coefficient $\alpha_k(t)$ decreases from one to zero, whereas the target coefficient $\beta_k(t)$ increases from zero to one. 
Unlike the standard linear path, the graph-spectral path assigns different interpolation schedules to different graph frequencies.
Specifically, larger values of $\omega_k$, corresponding to higher graph frequencies or stronger graph regularization, produce stronger contraction at intermediate flow times. 
The resulting path therefore introduces a graph-dependent structural inductive bias into flow matching.
Based on Theorem \ref{prop:graph_path}, we obtain the corresponding graph-spectral velocity field:
\begin{equation}
\dot{\hat{\boldsymbol{\gamma}}}_k^\star(t)=\dot{\alpha}_k(t)\hat{\mathbf{x}}_{0,k}+\dot{\beta}_k(t)\hat{\mathbf{x}}_{1,k},
\label{eq:spectral_velocity}
\end{equation}
where 
\begin{equation}
\dot{\alpha}_{k}(t)=-\frac{\omega_k\cosh(\omega_k(1-t))}{\sinh(\omega_k)}\quad \text{and}\quad
\dot{\beta}_{k}(t)=\frac{\omega_k\cosh(\omega_kt)}{\sinh(\omega_k)}.
\end{equation}
\begin{remark}
When $\omega_k=0$, $\alpha_k(t)$ and $\beta_k(t)$ are defined by their continuous limit, \textit{i.e.}, $\alpha_k(t)$ becomes $1-t$ and $\beta_k(t)$ becomes $t$.
The corresponding graph-frequency mode therefore degenerates to the standard linear path.
Since the normalized graph Laplacian always has at least one zero eigenvalue with its multiplicity equal to the number of connected components  \citep{chung1997spectral,shuman2013emerging}, the case for $\lambda_k=0$, \textit{i.e.}, $\omega_k=0$, always occurs for at least one graph-frequency mode.
\end{remark}
The graph-spectral path obtained in Theorem \ref{prop:graph_path} can be written in matrix form as
\begin{equation}
\hat{\boldsymbol{\Gamma}}^\star(t)=\mathbf{A}(t)\hat{\boldsymbol{\Gamma}}(0)+\mathbf{B}(t)\hat{\boldsymbol{\Gamma}}(1) 
\end{equation}
with
\begin{equation}
     \mathbf{A}(t)=\mathrm{diag}\left(\alpha_1(t),\ldots,\alpha_N(t)\right)\quad \text{and} \quad
     \mathbf{B}(t)= \mathrm{diag}\left(\beta_1(t),\ldots,\beta_N(t)\right).
\end{equation}
Performing the inverse graph Fourier transform gives the corresponding node-domain path:
\begin{equation}
\boldsymbol{\Gamma}^\star(t)=\boldsymbol{\Psi}\hat{\boldsymbol{\Gamma}}^\star(t)=\boldsymbol{\Psi}\mathbf{A}(t)\boldsymbol{\Psi}^\top\boldsymbol{\Gamma}(0)+\boldsymbol{\Psi}\mathbf{B}(t)\boldsymbol{\Psi}^\top\boldsymbol{\Gamma}(1)=\boldsymbol{\Psi}\mathbf{A}(t)\boldsymbol{\Psi}^\top\mathbf{X}_0+\boldsymbol{\Psi}\mathbf{B}(t)\boldsymbol{\Psi}^\top\mathbf{X}_1.  
\end{equation}
Then the corresponding node-domain velocity field is defined by
\begin{equation}
\dot{\boldsymbol{\Gamma}}^\star(t)=\boldsymbol{\Psi}\dot{\mathbf{A}}(t)\boldsymbol{\Psi}^\top\mathbf{X}_0+\boldsymbol{\Psi}\dot{\mathbf{B}}(t)\boldsymbol{\Psi}^\top\mathbf{X}_1, 
\end{equation}
where
\begin{equation}
     \dot{\mathbf{A}}(t)=\mathrm{diag}\left(\dot{\alpha}_1(t),\ldots,\dot{\alpha}_N(t)\right)\quad \text{and} \quad
     \dot{\mathbf{B}}(t)= \mathrm{diag}\left(\dot{\beta}_1(t),\ldots,\dot{\beta}_N(t)\right).
\end{equation}

\subsection{Model Training and Anomaly Scoring}
\label{sec:spectral_mixer}

We parameterize the velocity field associated with the graph-spectral path using a lightweight MLP architecture, similar to TSMixer as introduced in \citep{chen2023tsmixer}. 
The details of the model are provided in Appendix \ref{app:velocity_model}.
We train the velocity model in node domain with the regression objective:
\begin{equation}
\mathcal{L}_\mathsf{GRASP}(\boldsymbol{\theta}) = \mathbb{E}_{t \sim \mathcal{U}[0,1], \mathbf{X}_0\sim p_0, \mathbf{X}_1\sim p} \Big[
 \left\Vert \mathbf{V}_t(\mathbf{X}_t; {\boldsymbol{\theta}}) -\dot{\boldsymbol{\Gamma}}^\star(t) \right\Vert^2 \Big].
 \label{eq:loss}
\end{equation}
Suppose that the normal data follows the distribution $p$, we define the oracle marginal velocity in the node and graph Fourier domains as
\begin{equation}
\mathbf{V}^p_t(\mathbf{X})=\mathbb E[\dot{\boldsymbol{\Gamma}}^\star(t)\mid \mathbf{X}_t=\mathbf{X}] \quad \text{and} \quad \hat{\mathbf{V}}^p_t(\mathbf{X})=\mathbb E[\dot{\hat{\boldsymbol{\Gamma}}}^\star(t)\mid \mathbf{X}_t=\mathbf{X}].
\end{equation}
Similarly, the posterior means of the data endpoint in the node and graph Fourier domains are
\begin{equation}
\mathbf{M}^p_{t}(\mathbf{X})=\mathbb E[\boldsymbol{\Gamma}^\star(1)\mid \mathbf{X}_t=\mathbf{X}] \quad \text{and} \quad
\hat{\mathbf{M}}^p_{t}(\mathbf{X})=\mathbb E[\hat{\boldsymbol{\Gamma}}^\star(1)\mid \mathbf{X}_t=\mathbf{X}].
\end{equation}
Let $\hat{\mathbf{v}}^p_{t,k}$ and $\hat{\boldsymbol{\mu}}^p_{t,k}$ denote the $k$-th row of $\hat{\mathbf{V}}^p_t(\mathbf{X})$ and $\hat{\mathbf{M}}^p_{t}(\mathbf{X})$. Then we have the following result.

\begin{lemma}[Graph-spectral residual scaling]\label{prop:target_velocity_scaling}
For every $t\in(0,1)$ and graph-frequency mode $k$, the conditional velocity residual and the endpoint posterior residual in the graph Fourier domain satisfy
\begin{equation}
    \dot{\hat{\boldsymbol{\gamma}}}_k^\star(t)-\hat{\mathbf{v}}^p_{t,k}(\mathbf{X})=\rho_k(t)\left(\hat{\mathbf{x}}_{1,k}-\hat{\boldsymbol{\mu}}^p_{t,k}(\mathbf{X})\right) \quad \text{with}\quad\rho_k(t)=\frac{\omega_k}{\sinh(\omega_k(1-t))}.
\end{equation}
\end{lemma}
The proof of Lemma \ref{prop:target_velocity_scaling} is provided in Appendix \ref{proof:target_velocity_scaling}. Note that when $\omega_k=0$, $\rho_k(t)$ is defined by its continuous limit $\frac{1}{1-t}$.
Lemma \ref{prop:target_velocity_scaling} shows that the velocity residual rescales the endpoint residual by a factor that varies across both flow times and graph frequencies.

The conditional velocity residual in the graph Fourier domain can be written in matrix form as
\begin{equation}
    \dot{\hat{\boldsymbol{\Gamma}}}^\star(t)-\hat{\mathbf{V}}^p_{t}(\mathbf{X})=\mathbf{P}(t)\left(\hat{\mathbf{X}}_{1}-\hat{\mathbf{M}}^p_{t}(\mathbf{X})\right) \quad \text{with}  \quad \mathbf{P}(t)=\mathrm{diag}\left(\rho_1(t),\ldots,\rho_N(t)\right).
\end{equation}
Performing the inverse graph Fourier transform gives the node-domain conditional velocity residual:
\begin{equation}
    \dot{\boldsymbol{\Gamma}}^\star(t)-\mathbf{V}^p_{t}(\mathbf{X})=\boldsymbol{\Psi}\mathbf{P}(t)\boldsymbol{\Psi}^\top\left(\mathbf{X}_{1}-\mathbf{M}^p_{t}(\mathbf{X})\right).
\end{equation}
According to Theorem \ref{prop:graph_path}, we can see that the amount of endpoint information contained in an intermediate state also varies across flow times and graph frequencies. 
Consider a rescaled observation:
\begin{equation}
\frac{1}{\beta_k(t)}\hat{\boldsymbol{\gamma}}_k^\star(t)=\frac{\alpha_k(t)}{\beta_k(t)}\hat{\mathbf{x}}_{0,k}+\hat{\mathbf{x}}_{1,k}.
\end{equation}
Since the prior follows a standard Gaussian distribution, conditioning on the endpoint gives
\begin{equation}
    \frac{1}{\beta_k(t)}\hat{\boldsymbol{\gamma}}_k^\star(t)\mid \hat{\mathbf{x}}_{1,k}\sim\mathcal{N}\left(\hat{\mathbf{x}}_{1,k},\frac{\alpha_k^2(t)}{\beta_k^2(t)}\mathbf{I}\right).
\label{eq:condition_endpoint_Gaussian}
\end{equation}
The conditional variance in \eqref{eq:condition_endpoint_Gaussian} quantifies the effective noise level of the rescaled intermediate state with respect to the endpoint.
To account for both this noise level and the residual scaling elaborated in Lemma \ref{prop:target_velocity_scaling}, we define the anomaly score weight as the inverse effective noise variance multiplied by the inverse squared residual-scaling factor as follows:
\begin{equation}
    \eta_k(t)=\frac{\beta_k^2(t)}{\alpha_k^2(t)}\frac{1}{\rho^2_k(t)}=\frac{\sinh^2(\omega_kt)}{\sinh^2(\omega_k(1-t))}\frac{\sinh^2(\omega_k(1-t))}{\omega^2_k}=\frac{\sinh^2(\omega_kt)}{\omega^2_k}.
\end{equation}
\begin{wrapfigure}{r}{0.4\columnwidth}
    \centering
    \includegraphics[
        width=\linewidth
    ]{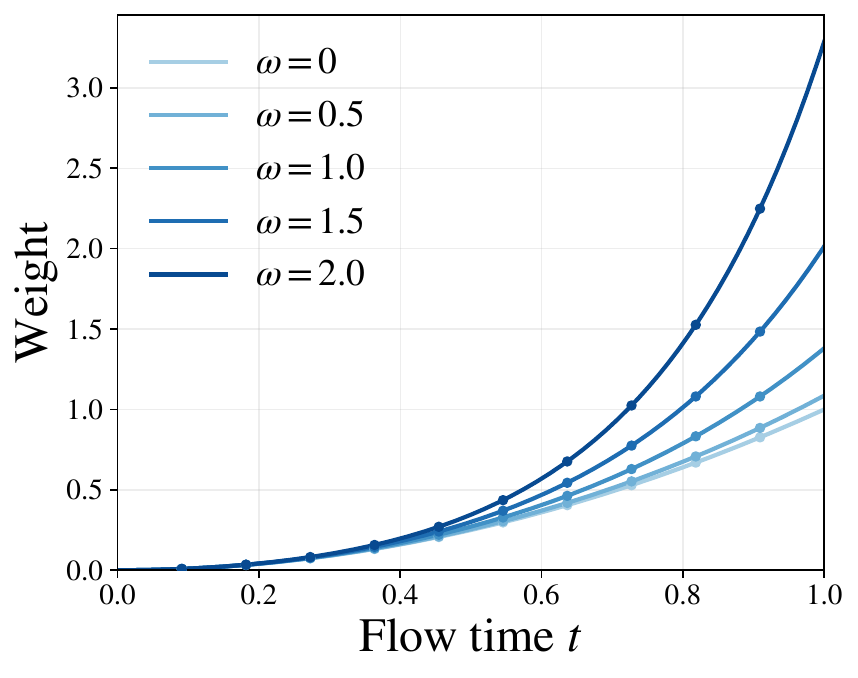}
    \vspace{-0.7cm}
    \caption{Weight schedule of $\eta_k(t)$.\label{fig:weight}
    }
\end{wrapfigure}
When $\omega_k=0$, the weight $\eta_k(t)$ is defined by its continuous limits $t^2$.
Figure \ref{fig:weight} illustrates the weighting schedule for different values of $\omega$.
From the figure, we observe that the weighting schedule assigns larger coefficients to velocity discrepancies evaluated closer to the data endpoint and at higher graph frequencies.
Given a test time window $\mathbf{X}_1$, we independently draw source samples $\mathbf{X}_0$ from the standard Gaussian distribution to form a finite prior set $\mathcal{X}_\mathrm{ad}$. 
Let $\mathcal{R}_\mathrm{ad}\subset(0,1)$ denote the set of flow times used for evaluation.
For each source sample and evaluation time, we construct the intermediate state and conditional target velocity using the graph-spectral path. 
We then aggregate the weighted squared velocity discrepancies across source samples, flow times, and graph-frequency modes to obtain the anomaly score as follows:
\begin{equation}
\xi(\mathcal{X}_\mathrm{ad},\mathcal{R}_\mathrm{ad},\mathbf{X}_1) = \sum_{k=1}^N\sum_{t\in\mathcal{R}_\mathrm{ad}}\sum_{\mathbf{X}_0\in\mathcal{X}_\mathrm{ad}} \eta_k(t)
 \left\Vert \hat{\mathbf{v}}_{t,k}(\mathbf{X}_t; {\boldsymbol{\theta}}) -\dot{\hat{\boldsymbol{\gamma}}}_k^\star(t) \right\Vert^2_2 .
 \label{eq:anomaly_score}
\end{equation}
A larger value of $\xi$ indicates a greater deviation from the velocity field learned from normal data and hence provide stronger evidence of anomalous behavior.
\subsection{Theoretical Analyses}
The graph Fourier basis is not unique when the graph Laplacian has repeated eigenvalues. 
More precisely, the eigenvectors within each repeated-eigenvalue eigenspace are identifiable only up to an orthogonal transformation \citep{agaskar2013spectral,sandryhaila2014discrete,deri2017spectral}.
Let $\tilde{\boldsymbol{\Psi}}=\boldsymbol{\Psi}\mathbf{Q}$ denote an alternative orthonormal eigenbasis, where $\mathbf{Q}$ is block orthogonal and acts only within eigenspaces associated with repeated eigenvalues.
The same graph Laplacian then admits the eigendecomposition
$
\mathbf{L}=\tilde{\boldsymbol{\Psi}}\mathbf{\Lambda}\tilde{\boldsymbol{\Psi}}^{\top}.
$
The following result shows that this potential ambiguity of the graph Fourier basis does not affect the velocity field model, the target velocity, and the anomaly score used by GRASP.

\begin{prop}[Eigenbasis invariance]\label{prop:eigenbasis_invariance}
In GRASP, the node-domain graph-spectral path, its conditional target velocity, and the anomaly score in \eqref{eq:anomaly_score} are all invariant to the choice of orthonormal Laplacian eigenbasis.
\end{prop}
The proof of Proposition \ref{prop:eigenbasis_invariance} is provided in Appendix \ref{proof:eigenbasis_invariance}.
Under the population-optimal squared flow matching objective, the learned velocity predictor satisfies
\begin{equation}
    \hat{\mathbf{V}}_{t}(\mathbf{X}_t; {\boldsymbol{\theta}})=\hat{\mathbf{V}}^p_t(\mathbf{X}_t).
\end{equation}
The trained velocity model approximates this population-optimal predictor.
We therefore analyze the anomaly score obtained with this oracle normal data velocity $\hat{\mathbf{V}}^p_t$.
Let $p_t(\hat{\mathbf{X}})$ and $q_t(\hat{\mathbf{X}})$ denote the graph-spectral path marginals induced by the normal endpoint distribution $p$ and test endpoint distribution $q$, respectively.
We define their spectral score matrices as
\begin{equation}
\hat{\mathbf{S}}_{t}^p(\hat{\mathbf{X}})=\nabla_{\hat{\mathbf{X}}}\log p_t(\hat{\mathbf{X}}),\quad \hat{\mathbf{S}}_{t}^q(\hat{\mathbf{X}})=\nabla_{\hat{\mathbf{X}}}\log q_t(\hat{\mathbf{X}}),
\end{equation}
whose $k$-th rows are denoted as $\hat{\mathbf{s}}_{t,k}^p(\hat{\mathbf{X}})$ and $\hat{\mathbf{s}}_{t,k}^q(\hat{\mathbf{X}})$, respectively.
We define $\hat{\boldsymbol{\mu}}^q_{t,k}$ analogously to $\hat{\boldsymbol{\mu}}^p_{t,k}$.
In the following, we show that the expected oracle anomaly score of GRASP admits a decomposition that contains a term that measures deviations between normal and test distributions.

\begin{theorem}[Decomposition of the expected oracle anomaly score]
\label{prop:anomaly_score}
Let $\mathcal X_{\mathrm{ad}}$ contain a finite set of i.i.d. source samples from $p_0$, independent of $\mathbf X_1$, and let
$\mathcal R_{\mathrm{ad}}\subset(0,1)$ be a finite set of evaluation times.
Consider the anomaly score in \eqref{eq:anomaly_score} evaluated using
the oracle normal-data velocity predictor $V_t^p$.
Its expectation over $\mathcal X_{\mathrm{ad}}$ and $\mathbf X_1\sim q$
decomposes as
\begin{equation}
\begin{aligned}
\mathbb E_{\mathcal X_{\mathrm{ad}},q}
\left[\xi(\mathcal X_{\mathrm{ad}},
\mathcal R_{\mathrm{ad}},\mathbf X_1)\right] =&
\sum_{k=1}^{N}\sum_{t\in\mathcal R_{\mathrm{ad}}}
|\mathcal{X}_\mathrm{ad}|\Biggl(
\frac{\sinh^2(\omega_k t)}
     {\sinh^2(\omega_k(1-t))}
\mathbb E_{p_0,q}
\left[
\left\|\hat{\mathbf x}_{1,k}
-\hat{\boldsymbol\mu}^{q}_{t,k}(\mathbf X)\right\|_2^2
\right] \\
&+
\frac{\sinh^2(\omega_k(1-t))}
     {\sinh^2(\omega_k)}
\mathbb E_{q_t}
\left[
\left\|\hat{\mathbf s}^{q}_{t,k}(\hat{\mathbf X})
-\hat{\mathbf s}^{p}_{t,k}(\hat{\mathbf X})\right\|_2^2
\right]
\Biggr),
\end{aligned}
\label{eq:anomaly_score_decomposition}
\end{equation}
where $\mathbb E_{p_0,q}$ denotes expectation over independent $\mathbf X_0\sim p_0$ and $\mathbf X_1\sim q$ and the coefficients are defined through their continuous limits when $\omega_k=0$.
\end{theorem}

\begin{corollary}[Boundedness of the endpoint-uncertainty term]
\label{prop:uncertainty_bound}
Under the setting of Theorem~\ref{prop:anomaly_score}, assume that
$\mathbb E_{\mathbf X_1\sim q}[\|\mathbf X_1\|_F^2]<\infty$.
Then, for any finite set of evaluation times
$\mathcal R_{\mathrm{ad}}\subset(0,1)$, the expected
endpoint-uncertainty term in
\eqref{eq:anomaly_score_decomposition} satisfies
\begin{equation}
\sum_{k=1}^{N}\sum_{t\in\mathcal R_{\mathrm{ad}}}
|\mathcal{X}_\mathrm{ad}|\frac{\sinh^2(\omega_k t)}
        {\sinh^2(\omega_k(1-t))}
\mathbb E_{p_0,q}
\left[
\left\|\hat{\mathbf x}_{1,k}
-\hat{\boldsymbol\mu}^{q}_{t,k}(\mathbf X)\right\|_2^2
\right]
\leq |\mathcal{X}_\mathrm{ad}||\mathcal R_{\mathrm{ad}}|NR.
\end{equation}
\end{corollary}

The proof of Theorem \ref{prop:anomaly_score} and Corollary \ref{prop:uncertainty_bound} are provided in Appendix \ref{proof:anomaly_score} and Appendix \ref{app:uncertainty_bound}, respectively.
The first term in \eqref{eq:anomaly_score_decomposition} captures graph-frequency-weighted irreducible endpoint uncertainty, which is bounded above as shown in Corollary \ref{prop:uncertainty_bound}.
Whereas the second term in \eqref{eq:anomaly_score_decomposition} measures the graph-frequency-weighted Fisher discrepancy between the test and normal path marginals, which provides essential signals for anomaly detection.
The decomposition therefore identifies a distribution-sensitive component of the proposed anomaly score that explicitly measures deviations between normal and test distributions.

\section{Experiments}

We evaluate GRASP on four widely used MTS anomaly detection datasets covering different application domains: the spacecraft telemetry dataset SMAP \citep{hundman2018detecting}, the IT infrastructure dataset SMD \citep{su2019robust}, the cybersecurity dataset CICIDS \citep{sharafaldin2018toward}, and the astrophysics dataset SWAN \citep{angryk2020swan}.
For each dataset, we construct the sensor graph using normal training data.
Specifically, we first compute a pairwise distance matrix $\boldsymbol{\Upsilon}\in\mathbb{R}^{N\times N}$ from the node features.
The distances are converted into edge weights between 0 and 1 using the Gaussian kernel $\exp(-(\frac{\upsilon_{ij}}{\iota})^2)$, where the decay rate $\iota$ is set to the standard deviation of $\boldsymbol{\Upsilon}$.
The resulting weight matrix is subsequently converted into a binary adjacency matrix with thresholding.

For each MTS, we use the first 80\% of the normal training data for anomaly detector optimization and the remaining 20\% for validation.
We evaluate MTS anomaly detection performance with three metrics: the area under the receiver operating characteristic curve (ROC), the area under the precision-recall curve (PRC), and the Best-F1 score.
ROC and PRC measure threshold-independent ranking performance, with PRC being particularly informative for highly imbalanced cases.
The Best-F1 score measures point-wise detection performance with the threshold that maximizes the F1 score using test labels.
It therefore represents retrospective detection potential rather than performance at a deployable threshold. 
For each dataset and random seed, we first average the results across its constituent series and then report the mean and standard deviation over ten random seeds.

\begin{table*}[t]
\centering
\caption{
Anomaly detection performance comparison.
}
\vspace{2pt}
\label{tab:main_results}
\resizebox{\textwidth}{!}{
\begin{tabular}{l|ccc|ccc|ccc|ccc}
\toprule
\multirow{2}{*}{Model}
& \multicolumn{3}{c|}{SMAP}
& \multicolumn{3}{c|}{SMD}
& \multicolumn{3}{c|}{CICIDS}
& \multicolumn{3}{c}{SWAN} \\
& PRC & ROC & Best-F1
& PRC & ROC & Best-F1
& PRC & ROC & Best-F1
& PRC & ROC & Best-F1 \\
\midrule

HBOS
& 0.1897 & 0.5650 & 0.2664 & 0.2678 & 0.7378 & 0.3556 & 0.2603 & 0.5586 & 0.3852 & 0.2792 & 0.5000 & 0.4365 \\

COPOD
& 0.1927 & 0.5918 & 0.2841 & 0.2139 & 0.7193 & 0.2951 & 0.2214 & 0.5566 & 0.3733 & 0.2792 & 0.5000 & 0.4365 \\

Autoencoder
& 0.2277 & 0.5859 & 0.2991 & 0.3613 & 0.7297 & 0.4214 & 0.3147 & \underline{0.7276} & 0.3938 & 0.3294 & 0.4944 & 0.4371 \\

USAD
& 0.2263 & 0.5850 & 0.3850 & 0.4105 & 0.8726 & 0.4810 & 0.2094 & 0.4619 & 0.3525 & 0.4405 & 0.5585 & 0.4365 \\

CNN
& 0.2136 & \underline{0.6577} & 0.3471 & 0.3730 & 0.7921 & 0.4313 & 0.2371 & 0.5543 & 0.3991 & 0.3760 & 0.4966 & 0.4365 \\

OmniAnomaly
& 0.2420 & 0.5950 & \underline{0.4065} & 0.4284 & \underline{0.8822} & \underline{0.5061} & 0.2138 & 0.4501 & 0.3548 & 0.4847 & 0.6130 & 0.4365 \\

TranAD
& 0.2122 & 0.6104 & 0.3264 & 0.3304 & 0.7691 & 0.4181 & 0.2043 & 0.4704 & 0.3416 & 0.3360 & 0.4911 & 0.4365 \\

A-Transformer
& 0.1340 & 0.5065 & 0.2060 & 0.0720 & 0.5013 & 0.1260 & 0.2396 & 0.4972 & 0.3060 & 0.2817 & 0.5021 & 0.4365 \\

TimesNet
& 0.1792 & 0.5336 & 0.3161 & 0.2839 & 0.8087 & 0.3713 & 0.1905 & 0.4657 & 0.3121 & 0.3068 & 0.4505 & 0.4365 \\

FITS
& 0.1455 & 0.5147 & 0.2798 & 0.2863 & 0.8299 & 0.4025 & 0.1849 & 0.4106 & 0.3126 & 0.3189 & 0.4607 & 0.4365 \\

GDN
& 0.2197 & 0.6099 & 0.3394 & 0.4006 & 0.8014 & 0.4640 & 0.2926 & 0.7017 & 0.3805 & 0.6423 & \underline{0.8356} & \underline{0.6432} \\

GCAD
& 0.2144 & 0.6348 & 0.3471 & 0.3569 & 0.8416 & 0.4535 & \underline{0.3636} & 0.7191 & 0.3844 & \underline{0.6445} & 0.8140 & 0.6231 \\

CATCH
& \underline{0.2450} & 0.6078 & 0.3664 & \underline{0.4763} & \textbf{0.8968} & 0.5057 & 0.2866 & 0.6162 & \underline{0.4023} & 0.4638 & 0.6218 & 0.4533 \\

\midrule
\rowcolor{gray!12}
\textbf{GRASP}
& \textbf{0.3149} & \textbf{0.7015} & \textbf{0.4178} & \textbf{0.4845} & 0.8567 & \textbf{0.5087} & \textbf{0.3996} & \textbf{0.8108} & \textbf{0.4615} & \textbf{0.7556} & \textbf{0.8758} & \textbf{0.6886} \\

\bottomrule
\end{tabular}
}
\end{table*}
We compare the proposed GRASP model with thirteen classical and deep-learning-based anomaly detection methods, including HBOS \citep{goldstein2012histogram}, COPOD \citep{li2020copod}, Autoencoder \citep{sakurada2014anomaly}, USAD \citep{audibert2020usad}, CNN \citep{munir2018deepant}, OmniAnomaly \citep{su2019robust}, TranAD \citep{TranAD},   A-Transformer \citep{xu2022anomaly}, TimesNet \citep{wu2023timesnet}, FITS \citep{xu2024fits}, GDN \citep{deng2021graph}, GCAD \citep{liu2025gcad}, and CATCH \citep{wu2025catch}.
For these baselines, we use the implementations provided in mTSBench \cite{zhou2026mtsbench} whenever available. For the rest of models, we use their official implementations.

In the following, we first evaluate the anomaly detection performance of the GRASP model and the baselines in Section \ref{sec:anomaly_detection_performance}.
Then we perform ablation studies to evaluate the individual contributions of the graph-spectral path, score weighting schedule, and data-dependent graph construction in Section \ref{sec:ablation}.
In Section \ref{sec:signal_attribution_analysis}, we further analyze how different flow times and graph frequencies contribute to the aggregated anomaly score.
Additional implementation details and experimental results are provided in Appendix \ref{app:experiments}.
The effects of different velocity-field architectures are studied in Appendix \ref{app:ablation_velocity_model}.
Appendix \ref{app:sensitivity_number_evaluations} analyzes how the model performs with different numbers of evaluated flow times and source samples, while Appendix \ref{app:sensitivity_tau} analyzes how the model performs with different graph Dirichlet energy weight $\tau$.
The inference efficiency of the model is discussed in Appendix \ref{app:inference_time}.

\subsection{Anomaly Detection Performance}\label{sec:anomaly_detection_performance}
Table \ref{tab:main_results} compares GRASP with the baselines on four datasets.
The best results are highlighted in bold, and the second-best results are underlined.
GRASP achieves the highest PRC and Best-F1 on all four datasets, as well as the highest ROC on three datasets.
Its PRC gains are particularly pronounced on SMAP and SWAN, indicating improved anomaly ranking across different application domains.
On SMD, however, CATCH achieves a higher ROC, although GRASP retains the best PRC and Best-F1.
Overall, these results show strong and consistent detection performance of GRASP.

\subsection{Ablation Studies}\label{sec:ablation}
\begin{wraptable}{r}{0.5\linewidth}
\centering
\vspace{-15pt}
\caption{Ablation results on SMAP and SMD.}
\vspace{2pt}
\label{tab:ablation}
\resizebox{\linewidth}{!}{
\begin{tabular}{l|ccc|ccc}
\toprule
\multirow{2}{*}{Model}
& \multicolumn{3}{c|}{SMAP}
& \multicolumn{3}{c}{SMD}\\
& PRC & ROC & Best-F1
& PRC & ROC & Best-F1\\
\midrule
GRASP-lin
& 0.2910 & 0.6879 & 0.3911
& 0.4762 & 0.8366 & 0.4968\\
GRASP-ER
& 0.2907 & 0.6909 & 0.3912
& 0.4797 & 0.8379 & 0.4993\\
GRASP-mean
& \underline{0.2915} & \underline{0.6941} & \underline{0.3920}
& \underline{0.4801} & \underline{0.8385} & \underline{0.5005}\\
\rowcolor{gray!12}
GRASP
& \textbf{0.3149} & \textbf{0.7015} & \textbf{0.4178}
& \textbf{0.4845} & \textbf{0.8567} & \textbf{0.5087}\\

\bottomrule
\end{tabular}
}
\vspace{-10pt}
\end{wraptable}
We conduct ablation studies using three variants of GRASP: 1) GRASP-mean retains the graph-spectral path with a uniform weighting schedule across flow times and graph frequencies; 2) GRASP-ER replaces the data-dependent graph in GRASP-mean with an Erd\H{o}s--R\'enyi graph that has the same sparsity level; 3) GRASP-lin replaces the graph-spectral path in GRASP-mean with the linear path by setting $\tau=0$.
The results on SMAP and SMD are presented in Table \ref{tab:ablation}, while the results on CICIDS and SWAN are deferred to Appendix \ref{app:ablation}.
GRASP achieves the best results on all reported metrics.
Its improvement over GRASP-mean supports the effectiveness of the proposed weighting schedule.
GRASP-mean consistently outperforms GRASP-lin, indicating that the graph-spectral path provides more informative anomaly signals than the linear path.
GRASP-mean also consistently outperforms GRASP-ER, suggesting that the data-dependent graph is more informative than a random graph.

\subsection{Flow-Time and Graph-Frequency Analysis}\label{sec:signal_attribution_analysis}
\begin{wrapfigure}{r}{0.6\columnwidth}
    \centering
    \vspace{-8pt}
    \includegraphics[
        width=\linewidth
    ]{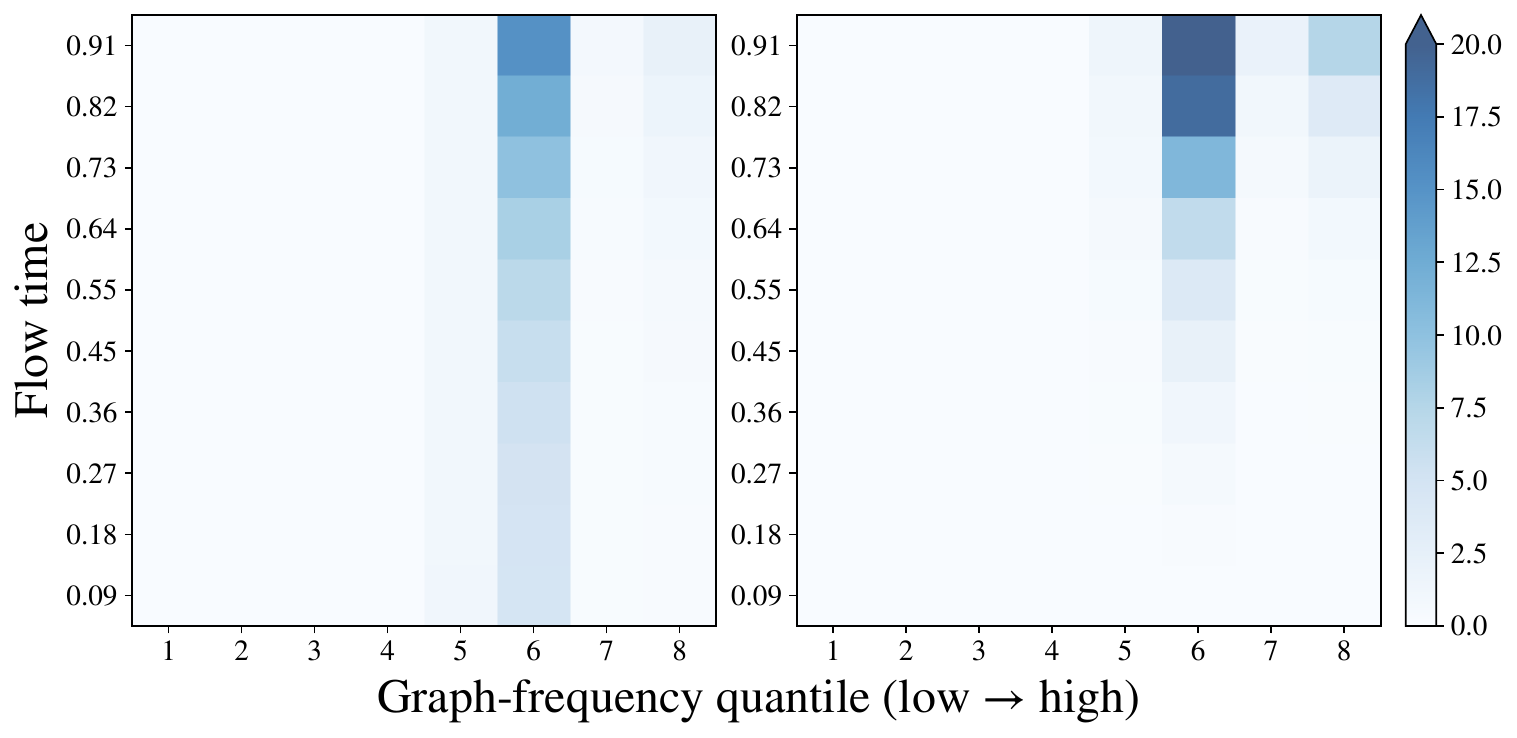}\vspace{-8pt}
    \caption{Unweighted (left) and weighted (right) anomaly signal attribution across time and frequency on SMAP. 
    }\label{fig:attribution_smap}
    \vspace{-8pt}
\end{wrapfigure}
Figure \ref{fig:attribution_smap} presents the unweighted and weighted anomaly signal attribution across flow times and graph-frequency quantiles on SMAP.
Specifically, we equally divide the graph-frequency spectrum into eight quantiles and aggregate the squared velocity residuals within each flow-time and graph-frequency bin.
Each heatmap is normalized separately to sum to $100\%$, so each cell reports its percentage of
the total anomaly signal attribution.
From the results, we observe that the anomaly evidence is concentrated toward the end of the flow trajectory and concentrated in a narrow graph-frequency band.
Applying the weighting schedule further emphasizes the anomaly signals in later flow times and larger graph frequencies.
This aligns with the fact that later states carry more information about the data endpoint than earlier states, which are more influenced by the Gaussian source. 
The dominant frequency band location differs across datasets (results on other datasets are deferred to Appendix \ref{app:attribution}), suggesting that anomaly-relevant graph frequencies depend on the specific systems.


\section{Conclusion}
In this work, we introduced a graph-informed flow matching approach for MTS anomaly detection named GRASP.
It relies on a graph-spectral path which is constructed by minimizing a fixed-endpoint action that integrates kinetic energy with graph Dirichlet energy.
GRASP detects anomalies by aggregating weighted velocity discrepancies across flow times, source samples, and graph frequencies.
We proved that the model is invariant to different Laplacian eigenbasis and the expected oracle anomaly score can be decomposed into bounded endpoint uncertainty and graph-frequency-weighted Fisher discrepancy.
Experiments on four datasets validated the effectiveness of GRASP.

Despite its promising results, GRASP has some limitations. 
First, it relies on a fixed graph, making it unsuitable when we have dynamically evolving graphs.
Moreover, computing the graph Fourier basis can become costly as the number of variables increases.
Future work could extend GRASP to be compatible with dynamic graphs and develop scalable approximations to graph spectral operations.
\newpage

\bibliography{iclr2027_conference}
\bibliographystyle{iclr2027_conference}

\newpage
\appendix

\doparttoc
\faketableofcontents

\addcontentsline{toc}{section}{Appendix}

\begingroup
\renewcommand{\partname}{}
\renewcommand{\thepart}{}
\part{{\Large Appendix}}
\endgroup

\parttoc
\newpage

\section{Related Work}\label{app:related_work}
\textbf{Graph-based MTS anomaly detection.}
Unsupervised MTS anomaly detection methods learn patterns of normal system operation and identify observations that deviate from them.
A common approach is to compute anomaly scores based on forecasting or reconstruction errors. For example, MTAD-GAT \citep{zhao2020multivariate} combines forecasting and reconstruction objectives during training and evaluates discrepancies between the observed signal and model outputs at inference. 
Graph-based approaches further incorporate dependencies among variables into these objectives. 
GDN \citep{deng2021graph} learns sensor relationships among variables to support forecasting-based detection, whereas DyEdgeGAT \citep{zhao2024dyedgegat} constructs input-dependent graphs for reconstruction-based detection. 
Other approaches characterize anomalies through association discrepancies \citep{xu2022anomaly}, differences in learned representations \citep{yang2023dcdetector}, deviations in the frequency domain \citep{wu2025catch}, or shifts in graph-spectral energy \cite{liu2026modeling}. More recently, diffusion-based methods have been introduced to reconstruct normal graph signals \citep{li2025diffgad}. However, these forecasting-, reconstruction-, and diffusion-based methods primarily evaluate discrepancies at the predicted or reconstructed endpoint. Consequently, subtle or partially predictable anomalies may remain difficult to detect when they produce only weak endpoint discrepancies.

\textbf{Flow matching and conditional path design.}
Flow matching learns a neural velocity field by regressing against conditional target velocities along prescribed probability paths that connect a simple source distribution to the data distribution. 
The choice of probability path is essential because it determines both the intermediate states and the target velocities used for training and evaluation \citep{du2026lagrangian}. 
Existing constructions, including rectified paths \citep{liu2023flow} and optimal-transport-based paths \citep{tong2024improving}, typically use straight-line interpolation between coupled endpoints.
From a variational perspective, the standard linear conditional path minimizes the action associated with a kinetic-energy Lagrangian under fixed-endpoint constraints \citep{benamou2000computational,villani2009optimal}, and it has become a standard choice in flow matching \citep{lipman2023flow}.
However, this construction depends only on the endpoints and does not explicitly incorporate the geometry or relational structure of the data. The resulting intermediate states may therefore be poorly aligned with data supported on non-Euclidean domains \citep{fang2026escaping}.

\textbf{Graph-aware generative modeling.} 
Several recent methods investigate incorporating graph structural information into diffusion or other generative frameworks.
Graph-aware diffusion models introduce structural information through graph-based denoising architectures \citep{uslu2026graph} or graph-aware forward noising schedules \citep{rozada2026graph}.
\cite{yang2025topological} extends Schr\"odinger bridge matching to topological domains such as graphs and simplicial complexes, while \cite{wyrwal2026topological} incorporates topological information into the reference process through a Laplacian-derived drift.
These methods introduce graph structure into the denoising model or a predefined stochastic reference process. In contrast, GRASP derives the conditional probability path directly as the solution to a fixed-endpoint variational problem that combines kinetic energy with graph Dirichlet energy. This formulation yields a closed-form, graph-frequency-dependent path specifically designed for velocity-based MTS anomaly detection.

\section{Notation Summary}
\label{app:notation}

Table~\ref{tab:notation} summarizes the notation used throughout the
paper.

\begin{table}[t]
\centering
\caption{Summary of notation used throughout the paper.}\vspace{2pt}
\label{tab:notation}
\small
\renewcommand{\arraystretch}{1.08}
\begin{tabularx}{\textwidth}{@{}lX@{}}
\toprule
\textbf{Symbol} & \textbf{Description} \\
\midrule
$N,\ R$ & Number of variables and length of a time-series window. \\
$\mathcal G=(\mathcal N,\mathcal E)$ & Sensor graph with node set $\mathcal N$ and edge set $\mathcal E$. \\
$\mathbf X\in\mathbb R^{N\times R}$ & Multivariate time-series window. \\
$\mathbf L$ & Symmetric normalized graph Laplacian. \\
$\boldsymbol{\Psi},\boldsymbol\Lambda$ & Laplacian eigenvector and eigenvalue matrices satisfying $\mathbf L=\boldsymbol\Psi\boldsymbol\Lambda\boldsymbol\Psi^\top$. \\
$\lambda_k,\boldsymbol\psi_k$ & Laplacian eigenvalue and eigenvector associated with graph-frequency mode $k$. \\
$\hat{\mathbf X}=\boldsymbol\Psi^\top\mathbf X$ & Graph Fourier transform of $\mathbf X$. \\
$\hat{\mathbf x}_{t,k}$ & The $k$-th row of $\hat{\mathbf X}$, representing graph-frequency mode $k$ across the time window. \\
$t\in[0,1]$ & Flow time. \\
$\mathbf X_0,\mathbf X_1,\mathbf X_t$ & Source endpoint, data endpoint, and intermediate state. \\
$p_0$&Standard Gaussian source distribution.\\
$p,\ q$ & Normal and test endpoint distributions. \\
$p_t,\ q_t$ & Intermediate-state distributions induced by endpoint distributions $p$ and $q$. \\
$\phi_t$ & Conditional interpolation map between $\mathbf X_0$ and $\mathbf X_1$. \\
$\mathbf U_t$ & Conditional target velocity induced by $\phi_t$. \\
$\mathbf V_t(\mathbf X;\boldsymbol\theta)$ & Learned time-dependent velocity field. \\
$\boldsymbol\Gamma^\star(t),\dot{\boldsymbol\Gamma}^\star(t)$ & Optimal graph-spectral path and its velocity. \\
$\hat{\boldsymbol\gamma}_k^\star(t)$ & The $k$-th graph-frequency component of $\Gamma^\star(t)$. \\
$\tau$ & Weight of the graph Dirichlet energy. \\
$\omega_k=\sqrt{\tau\lambda_k}$ & Graph-spectral parameter associated with graph-frequency mode $k$. \\
$\alpha_k(t),\beta_k(t)$ & Source and target interpolation coefficients for graph-frequency mode $k$. \\
$\mathbf A(t),\mathbf B(t)$ & Diagonal matrices collecting $\alpha_k(t)$ and $\beta_k(t)$. \\
$\mathbf V_t^p(\mathbf X)$ & Oracle marginal velocity field induced by the normal endpoint distribution. \\
$\hat{\mathbf v}_{t,k}^p$ & The $k$-th row of $\mathbf V_t^p(\mathbf X)$. \\
$\mathbf M_t^p(\mathbf X),\mathbf M_t^q(\mathbf X)$ & Endpoint posterior means under $p$ and $q$. \\
$\hat{\boldsymbol\mu}_{t,k}^{p},\hat{\boldsymbol\mu}_{t,k}^{q}$ & Endpoint posterior mean of graph-frequency mode $k$ under distribution $p$ and $q$. \\
$\rho_k(t)$ & Scaling factor relating the conditional velocity residual to the endpoint posterior residual. \\
$\eta_k(t)$ & Anomaly-score weight for graph-frequency mode $k$ at flow-time $t$. \\
$\mathcal X_{\mathrm{ad}},\mathcal R_{\mathrm{ad}}$ & Sets of source samples and evaluation flow times used for anomaly scoring.\\
$\xi(\mathcal X_{\mathrm{ad}},\mathcal R_{\mathrm{ad}},\mathbf X_1)$ & Anomaly score of test window $\mathbf X_1$. \\
$\hat{\mathbf S}_{t}^p,\hat{\mathbf S}_{t}^q$ & Spectral score matrices of the normal and test path marginals. \\
$\mathbf d_{t,k}$ & Endpoint posterior residual for graph-frequency mode $k$. \\
$\boldsymbol\delta_{t,k}$ & Difference between the test and normal spectral scores of graph-frequency mode $k$. \\
$\tilde{\boldsymbol\Psi}=\boldsymbol\Psi\mathbf Q$ & Alternative Laplacian eigenbasis within repeated-eigenvalue eigenspaces. \\
$\mathbf Q$ & Block-diagonal orthogonal transformation matrix acting within repeated-eigenvalue eigenspaces. \\

\bottomrule
\end{tabularx}
\end{table}

\section{Velocity-Field Architecture}\label{app:velocity_model}
Given an intermediate state
$\mathbf{X}_t\in\mathbb{R}^{N\times R}$ and its flow time $t$, we first
encode $t$ using a sinusoidal flow-time embedding function $\mathbf{e}(t)$.
The intermediate state and time embedding are then concatenated and mapped to an initial hidden representation through an input projection:
\begin{equation}
    \mathbf{H}^{(0)}
    =
    \phi_{\mathrm{in}}
    \left(
        \mathbf{X}_t,\, e(t)
    \right),
\end{equation}
where $\phi_{\mathrm{in}}$ denote the input projection. 

Following the TSMixer architecture \citep{chen2023tsmixer}, we process the hidden representation using $L_\mathrm{TSMixer}$ residual mixing blocks. 
Each block alternates between temporal mixing and cross-variable mixing as follows:
\begin{align}
    \tilde{\mathbf{H}}^{(\ell)}
    &=
    \mathbf{H}^{(\ell-1)}
    +
    \mathcal{M}_{\mathrm{time}}^{(\ell)}
    \left(
        \operatorname{Norm}
        (\mathbf{H}^{(\ell-1)})
    \right),
    \label{eq:tsmixer_time}
    \\
    \mathbf{H}^{(\ell)}
    &=
    \tilde{\mathbf{H}}^{(\ell)}
    +
    \mathcal{M}_{\mathrm{var}}^{(\ell)}
    \left(
        \operatorname{Norm}
        (\tilde{\mathbf{H}}^{(\ell)})
    \right),
    \label{eq:tsmixer_var}
\end{align}
for $\ell=1,\ldots,L_\mathrm{TSMixer}$.
Here, $\operatorname{Norm}$ represents layer normalization \citep{ba2016layer}, $\mathcal{M}_{\mathrm{time}}^{(\ell)}$ is an MLP applied along the temporal
dimension and shared across variables, while $\mathcal{M}_{\mathrm{var}}^{(\ell)}$ is an MLP applied along the variable dimension and shared across time steps. Each mixing MLP consists of two linear
layers with a nonlinear activation function ReLU and dropout.

Finally, an output projection maps the hidden representation back to the original data space:
\begin{equation}
    \mathbf{V}_t(\mathbf{X}_t;\boldsymbol{\theta})
    =
    \phi_{\mathrm{out}}
    \left(
        \mathbf{H}^{(L_\mathrm{TSMixer})}
    \right)
    \in\mathbb{R}^{N\times R}.
    \label{eq:velocity_network}
\end{equation}
The alternating mixing operations enable the velocity predictor to capture temporal dependencies within individual variables and interactions across variables while retaining a lightweight, fully MLP-based architecture. Importantly, in GRASP, graph structure is incorporated through the probability path and conditional target velocity rather than through the velocity-field architecture itself.

\section{Proofs of Theoretical Results}\label{app:proofs}
\subsection{Proof of Theorem \ref{prop:graph_path}\label{proof:graph_path}}
\begin{proof}
The Euler-Lagrange equation corresponds to the variational problem in \eqref{eq:separated_variational_problem} is
\begin{equation}
\omega^2_k\hat{\boldsymbol{\gamma}}_{k}(t)-\ddot{\hat{\boldsymbol{\gamma}}}_{k}(t)=\mathbf{0},
\end{equation}
whose general solution is
\begin{equation}
\hat{\boldsymbol{\gamma}}_k(t)=\mathbf{c}_{1,k}e^{\omega_k t}+\mathbf{c}_{2,k}e^{-\omega_k t}
\end{equation}
with some $\mathbf{c}_{1,k}$ and $\mathbf{c}_{2,k}$.
Since the hyperbolic functions satisfy
\begin{equation}
\sinh(\omega_kt)=\frac{e^{\omega_k t}-e^{-\omega_k t}}{2},\quad \cosh(\omega_kt)=\frac{e^{\omega_k t}+e^{-\omega_k t}}{2},
\end{equation}
we can further transform it into
\begin{equation}
\hat{\boldsymbol{\gamma}}_k(t)=\left(\mathbf{c}_{1,k}-\mathbf{c}_{2,k}\right)\sinh(\omega_kt)+\left(\mathbf{c}_{1,k}+\mathbf{c}_{2,k}\right)\cosh(\omega_kt).
\end{equation}

Considering the boundary condition, we have
\begin{equation}
\hat{\boldsymbol{\gamma}}_k^\star(t)=\frac{\hat{\mathbf{x}}_{1,k}-\cosh(\omega_k)\hat{\mathbf{x}}_{0,k}}{\sinh(\omega_k)}\sinh(\omega_kt)+\cosh(\omega_kt)\hat{\mathbf{x}}_{0,k},
\end{equation}
which can be rewritten as
\begin{equation}
\begin{aligned}
\hat{\boldsymbol{\gamma}}_k^\star(t)&=\left(\cosh(\omega_kt)-\frac{\cosh(\omega_k)\sinh(\omega_kt)}{\sinh(\omega_k)}\right)\hat{\mathbf{x}}_{0,k}+\frac{\sinh(\omega_kt)}{\sinh(\omega_k)}\hat{\mathbf{x}}_{1,k}\\
&=\frac{\sinh(\omega_k(1-t))}{\sinh(\omega_k)}\hat{\mathbf{x}}_{0,k}+\frac{\sinh(\omega_kt)}{\sinh(\omega_k)}\hat{\mathbf{x}}_{1,k}.  
\end{aligned}
\label{eq:graph_spectral_path}
\end{equation}
When $\omega_k=0$, the Euler-Lagrange equation reduces to
\begin{equation}
\ddot{\hat{\boldsymbol{\gamma}}}_{k}(t)=\mathbf{0}.
\end{equation}
Under the same boundary conditions, its solution is
\begin{equation}
\hat{\boldsymbol{\gamma}}_k^\star(t)=(1-t)\hat{\mathbf{x}}_{0,k}+t\hat{\mathbf{x}}_{1,k},
\end{equation}
which coincides with the continuous limit of \eqref{eq:graph_spectral_path} as $\omega_k\rightarrow0$.

The kinetic-energy term in \eqref{eq:separated_variational_problem} is strictly convex over the set of paths satisfying the fixed-endpoint constraints. Indeed, two admissible paths with identical derivatives can differ only by a constant, which must be zero because they share the same endpoints. Moreover, the graph-frequency regularization term is convex since \(\omega_k^2=\tau\lambda_k\geq 0\). Therefore, the complete objective in \eqref{eq:separated_variational_problem} is strictly convex over its feasible set. Since the path in \eqref{eq:graph_spectral_path} satisfies both the Euler–Lagrange equation and the boundary conditions, it is the unique global minimizer of the variational problem defined in \eqref{eq:separated_variational_problem}, through which the proof is completed.
\end{proof}

\subsection{Proof of Lemma \ref{prop:target_velocity_scaling}\label{proof:target_velocity_scaling}}

\begin{proof}
According to Theorem \ref{prop:graph_path}, the $k$-th graph-frequency component of the conditional path is
\begin{equation}
\hat{\boldsymbol{\gamma}}_k^\star(t)=\alpha_k(t)\hat{\mathbf{x}}_{0,k}+\beta_k(t)\hat{\mathbf{x}}_{1,k}.
\end{equation}
For $t\in(0,1)$, we have $\alpha_k(t)>0$. Then, we can obtain
\begin{equation}
\hat{\mathbf{x}}_{0,k}=\frac{1}{\alpha_k(t)}\left(\hat{\boldsymbol{\gamma}}_k^\star(t)-\beta_k(t)\hat{\mathbf{x}}_{1,k}\right).  
\label{eq:x_0_reexpressed}
\end{equation}
Substituting \eqref{eq:x_0_reexpressed} into the conditional target velocity in \eqref{eq:spectral_velocity} gives
\begin{equation}
\dot{\hat{\boldsymbol{\gamma}}}_k^\star(t)=\frac{\dot{\alpha}_k(t)}{\alpha_k(t)}\left(\hat{\boldsymbol{\gamma}}_k^\star(t)-\beta_k(t)\hat{\mathbf{x}}_{1,k}\right)+\dot{\beta}_k(t)\hat{\mathbf{x}}_{1,k}. 
\end{equation}
which can be rewritten as
\begin{equation}
\dot{\hat{\boldsymbol{\gamma}}}_k^\star(t)=\frac{\dot{\alpha}_k(t)}{\alpha_k(t)}\hat{\boldsymbol{\gamma}}_k^\star(t)+\rho_k(t)\hat{\mathbf{x}}_{1,k}
\label{eq:target_velocity_reform}
\end{equation}
with
\begin{equation}
    \rho_k(t)=\dot{\beta}_k(t)-\frac{\dot{\alpha}_k(t)}{\alpha_k(t)}\beta_k(t).
\end{equation}
For $\omega_k>$, substituting the expressions for $\alpha_k(t)$ and $\beta_k(t)$ yields
\begin{equation}
\begin{aligned}
    \rho_k(t)&=\frac{\omega_k\cosh(\omega_kt)}{\sinh(\omega_k)}+\frac{\omega_k\cosh(\omega_k(1-t))}{\sinh(\omega_k)}\frac{\sinh(\omega_k)}{\sinh(\omega_k(1-t))}\frac{\sinh(\omega_kt)}{\sinh(\omega_k)}\\
    &=\frac{\omega_k}{\sinh(\omega_k)}\frac{\cosh(\omega_kt)\sinh(\omega_k(1-t))+\cosh(\omega_k(1-t))\sinh(\omega_kt)}{\sinh(\omega_k(1-t))}\\
    &=\frac{\omega_k}{\sinh(\omega_k)}\frac{\sinh(\omega_kt+\omega_k(1-t))}{\sinh(\omega_k(1-t))}\\
    &=\frac{\omega_k}{\sinh(\omega_k(1-t))}.
\end{aligned}
\end{equation}
Based on the definition of the oracle marginal velocity, we have 
\begin{equation}
\begin{aligned}
\hat{\mathbf{v}}^p_{t,k}(\mathbf{X})&=\mathbb E[\dot{\hat{\boldsymbol{\gamma}}}_k^\star(t)\mid \mathbf{X}_t=\mathbf{X}]\\
&=\mathbb E[\frac{\dot{\alpha}_k(t)}{\alpha_k(t)}\hat{\boldsymbol{\gamma}}_k^\star(t)+\rho_k(t)\hat{\mathbf{x}}_{1,k}\mid \mathbf{X}_t=\mathbf{X}]\\
&=\mathbb E[\frac{\dot{\alpha}_k(t)}{\alpha_k(t)}\hat{\boldsymbol{\gamma}}_k^\star(t)\mid \mathbf{X}_t=\mathbf{X}]+\mathbb E[\rho_k(t)\hat{\mathbf{x}}_{1,k}\mid \mathbf{X}_t=\mathbf{X}]\\
&=\frac{\dot{\alpha}_k(t)}{\alpha_k(t)}\hat{\boldsymbol{\gamma}}_k^\star(t)+\rho_k(t)\hat{\boldsymbol{\mu}}^p_{t,k}.
\end{aligned}
\label{eq:oracle_marginal_velocity}
\end{equation}
Subtracting \eqref{eq:oracle_marginal_velocity} from \eqref{eq:target_velocity_reform}, we obtain
\begin{equation}
\begin{aligned}
    \dot{\hat{\boldsymbol{\gamma}}}_k^\star(t)-\hat{\mathbf{v}}^p_{t,k}(\mathbf{X})&=\frac{\dot{\alpha}_k(t)}{\alpha_k(t)}\hat{\boldsymbol{\gamma}}_k^\star(t)+\rho_k(t)\hat{\mathbf{x}}_{1,k}-\left(\frac{\dot{\alpha}_k(t)}{\alpha_k(t)}\hat{\boldsymbol{\gamma}}_k^\star(t)+\rho_k(t)\hat{\boldsymbol{\mu}}^p_{t,k}(\mathbf{X})\right)\\
    &=\rho_k(t)\left(\hat{\mathbf{x}}_{1,k}-\hat{\boldsymbol{\mu}}^p_{t,k}(\mathbf{X})\right),   
\end{aligned}
\end{equation}
which completes the proof.
\end{proof}

\subsection{Proof of Proposition \ref{prop:eigenbasis_invariance}\label{proof:eigenbasis_invariance}}
\begin{proof}
Let $\nu_1,\ldots,\nu_J$ denote the distinct eigenvalues of
$\mathbf{L}$, and define
\begin{equation}
 \mathcal{I}_j
=
\{k:\lambda_k=\nu_j\}   
\end{equation}
as the index set associated with eigenvalue $\nu_j$. Any alternative orthonormal eigenbasis of $\mathbf{L}$ with the same eigenvalue ordering can be written as
\begin{equation}
   \tilde{\boldsymbol{\Psi}}
=
\boldsymbol{\Psi}\mathbf{Q},
\qquad
\mathbf{Q}
=
\operatorname{blkdiag}
\left(
\mathbf{Q}_1,\ldots,\mathbf{Q}_J
\right), 
\end{equation}

where each
$\mathbf{Q}_j\in
\mathbb{R}^{|\mathcal{I}_j|\times|\mathcal{I}_j|}$
is orthogonal. Hence, we have
\begin{equation}
 \mathbf{Q}^{\top}\mathbf{Q}
=
\mathbf{Q}\mathbf{Q}^{\top}
=
\mathbf{I},
\qquad
\mathbf{Q}\mathbf{\Lambda}
=
\mathbf{\Lambda}\mathbf{Q}.   
\end{equation}

All graph-spectral interpolation coefficients depend on the graph-frequency mode $k$ only through $\lambda_k$.
Therefore, their values are identical for all modes within the same repeated-eigenvalue eigenspace.
It follows that
\begin{equation}
    \mathbf{Q}\mathbf{A}(t)
=
\mathbf{A}(t)\mathbf{Q},
\qquad
\mathbf{Q}\mathbf{B}(t)
=
\mathbf{B}(t)\mathbf{Q},
\end{equation}
and similarly,
\begin{equation}
    \mathbf{Q}\dot{\mathbf{A}}(t)
=
\dot{\mathbf{A}}(t)\mathbf{Q},
\qquad
\mathbf{Q}\dot{\mathbf{B}}(t)
=
\dot{\mathbf{B}}(t)\mathbf{Q}.
\end{equation}
Under the alternative eigenbasis, the node-domain graph-spectral path satisfies
\begin{equation}
\begin{aligned}
\dot{\mathbf{\Gamma}}^{\star}(t)
&=
\boldsymbol{\Psi}\dot{\mathbf{A}}(t)\boldsymbol{\Psi}^{\top}\mathbf{X}_0
+
\boldsymbol{\Psi}\dot{\mathbf{B}}(t)\boldsymbol{\Psi}^{\top}\mathbf{X}_1
\\
&=
\boldsymbol{\Psi}\mathbf{Q}
\dot{\mathbf{A}}(t)
\mathbf{Q}^{\top}\boldsymbol{\Psi}^{\top}\mathbf{X}_0
+
\boldsymbol{\Psi}\mathbf{Q}
\dot{\mathbf{B}}(t)
\mathbf{Q}^{\top}\boldsymbol{\Psi}^{\top}\mathbf{X}_1
\\
&=
\tilde{\boldsymbol{\Psi}}
\dot{\mathbf{A}}(t)
\tilde{\boldsymbol{\Psi}}^{\top}\mathbf{X}_0
+
\tilde{\boldsymbol{\Psi}}
\dot{\mathbf{B}}(t)
\tilde{\boldsymbol{\Psi}}^{\top}\mathbf{X}_1.
\end{aligned}    
\end{equation}
Thus, the node-domain conditional target velocity is invariant to the particular choice of eigenvectors within each degenerate eigenspace.

The learned velocity field model $\mathbf{V}_t(\mathbf{X}_t;\boldsymbol{\theta})$ is defined entirely in the node domain as a function of $\mathbf{X}_t$, which does not directly depend on $\boldsymbol{\Psi}$. Therefore, for fixed parameters \(\boldsymbol{\theta}\), its output is unaffected by the choice of Laplacian eigenbasis. Moreover, because the node-domain conditional target velocity is invariant, the training objective in \eqref{eq:loss} is also eigenbasis invariant. 

It remains to establish the invariance of the anomaly score. Define the node-domain velocity residual as
\begin{equation}
 \mathbf{R}_t
=
\mathbf{V}_t(\mathbf{X}_t;\boldsymbol{\theta})
-
\dot{\mathbf{\Gamma}}^{\star}(t)
\in\mathbb{R}^{N\times R},   
\end{equation}
and its graph Fourier representation
\begin{equation}
 \hat{\mathbf{R}}_t
=
\boldsymbol{\Psi}^{\top}\mathbf{R}_t.   
\end{equation}
The $k$-th row of $\hat{\mathbf{R}}_t$ is
\begin{equation}
\hat{\mathbf{r}}_{t,k}
=
\hat{\mathbf{v}}_{t,k}
-
\dot{\hat{\boldsymbol{\gamma}}}^{\star}_k(t).    
\end{equation}
Furthermore, define the diagonal weighting matrix
\begin{equation}
\mathbf{H}(t)
=
\operatorname{diag}
\left(
\eta_1(t),\ldots,\eta_N(t)
\right).    
\end{equation}
Then, for a fixed $t$ and $\mathbf{X}_0$, the corresponding contribution
to the anomaly score can be written as
\begin{equation}
\begin{aligned}
\sum_{k=1}^{N}
\eta_k(t)
\left\|
\hat{\mathbf{v}}_{t,k}
-
\dot{\hat{\boldsymbol{\gamma}}}^{\star}_k(t)
\right\|_2^2
&=
\sum_{k=1}^{N}
\eta_k(t)
\left\|
\hat{\mathbf{r}}_{t,k}
\right\|_2^2
\\
&=
\operatorname{Tr}
\left(
\hat{\mathbf{R}}_t^{\top}
\mathbf{H}(t)
\hat{\mathbf{R}}_t
\right).
\end{aligned}    
\label{eq:eigenbasis_weighted_residual}
\end{equation}

Under the alternative eigenbasis
$\tilde{\boldsymbol{\Psi}}=\boldsymbol{\Psi}\mathbf{Q}$, the corresponding
spectral residual is
\begin{equation}
\tilde{\mathbf{R}}_t
=
\tilde{\boldsymbol{\Psi}}^{\top}\mathbf{R}_t
=
\mathbf{Q}^{\top}\hat{\mathbf{R}}_t.    
\end{equation}
Moreover, from the definition of the anomaly weight,
\begin{equation}
 \eta_k(t)
=
\frac{\sinh^2(\omega_k t)}{\omega_k^2},
\qquad
\omega_k=\sqrt{\tau\lambda_k},   
\end{equation}
so $\eta_k(t)$ also depends on $k$ only through $\lambda_k$.
Consequently, all modes belonging to the same eigenspace have the same
weight. Hence, $\mathbf{H}(t)$ is scalar within each degenerate
eigenspace and satisfies
\begin{equation}
  \mathbf{Q}\mathbf{H}(t)
=
\mathbf{H}(t)\mathbf{Q},
\qquad
\mathbf{Q}\mathbf{H}(t)\mathbf{Q}^{\top}
=
\mathbf{H}(t).  
\end{equation}
Therefore,
\begin{align}
\operatorname{Tr}
\left(
\tilde{\mathbf{R}}_t^{\top}
\mathbf{H}(t)
\tilde{\mathbf{R}}_t
\right)
&=
\operatorname{Tr}
\left(
\hat{\mathbf{R}}_t^{\top}
\mathbf{Q}\mathbf{H}(t)\mathbf{Q}^{\top}
\hat{\mathbf{R}}_t
\right)\\
&=
\operatorname{Tr}
\left(
\hat{\mathbf{R}}_t^{\top}
\mathbf{H}(t)
\hat{\mathbf{R}}_t
\right).
\end{align}
Thus, although individual spectral coordinates within a degenerate
eigenspace depend on the choice of eigenbasis, their weighted aggregate
contribution to the anomaly score is invariant.

Finally, summing the above equality over
$t\in\mathcal{R}_{\mathrm{ad}}$ and
$\mathbf{X}_0\in\mathcal{X}_{\mathrm{ad}}$ shows that
$
\xi(\mathcal{X}_{\mathrm{ad}},
    \mathcal{R}_{\mathrm{ad}},
    \mathbf{X}_1)
$
is unchanged under any orthogonal rotation within a repeated
eigenspace. Therefore, the velocity field model, the conditional target velocity,
and the resulting anomaly score are all invariant to the choice of
Laplacian eigenbasis.
\end{proof}

\subsection{Proof of Theorem \ref{prop:anomaly_score}\label{proof:anomaly_score}}
\begin{proof}
According to Theorem \ref{prop:graph_path}, we have
\begin{equation}
\hat{\boldsymbol{\gamma}}_k^\star(t)=\alpha_k(t)\hat{\mathbf{x}}_{0,k}+\beta_k(t)\hat{\mathbf{x}}_{1,k}.
\end{equation}
Since the graph Fourier transform is orthonormal and \(\mathbf X_0\sim\mathcal N(\mathbf 0,\mathbf I)\), each source component satisfies $\hat{\mathbf x}_{0,k}\sim\mathcal N(\mathbf 0,\mathbf I).$ Thus, conditioning on the data endpoint gives
\begin{equation}
    \hat{\boldsymbol{\gamma}}_k^\star(t)\mid \hat{\mathbf{x}}_{1,k}\sim\mathcal{N}\left(\beta_k(t)\hat{\mathbf{x}}_{1,k},\alpha_k^2(t)\mathbf{I}\right).
\end{equation}
The conditional density is
\begin{equation}
    \phi_t(\hat{\mathbf{x}}_k(t)\mid \hat{\mathbf{x}}_{1,k})=\frac{1}{(2\pi\alpha^2_k(t))^{R/2}}\exp\left[-\frac{\|\hat{\mathbf{x}}_k-\beta_k(t)\hat{\mathbf{x}}_{1,k}\|_2^2}{2\alpha_k^2(t)}\right].
\end{equation}
For \(t\in(0,1)\), both \(\alpha_k(t)\) and \(\beta_k(t)\) are positive. The score over the $k$-th graph-frequency is
\begin{equation}
\begin{aligned}
    \hat{\mathbf{s}}_{t,k}^p(\hat{\mathbf{X}})&=\mathbb{E}_{p_0,p}\left[\nabla_{\hat{\mathbf{x}}_k}\log p_t(\hat{\mathbf{X}}\mid\hat{\mathbf{X}}_1)\mid\hat{\mathbf{X}}_t=\hat{\mathbf{X}}\right]\\
    &=\mathbb{E}_{p_0,p}\left[-\frac{\hat{\mathbf{x}}_k-\beta_k(t)\hat{\mathbf{x}}_{1,k}}{\alpha_k^2(t)}\mid\hat{\mathbf{X}}_t=\hat{\mathbf{X}}\right]\\
    &=-\frac{\hat{\mathbf{x}}_k}{\alpha_k^2(t)}+\frac{\beta_k(t)}{\alpha_k^2(t)}\hat{\boldsymbol{\mu}}^p_{t,k}.
\end{aligned}
\end{equation}
Therefore, we have
\begin{equation}
    \beta_k(t)\hat{\boldsymbol{\mu}}^p_{t,k}(\mathbf{X})=\hat{\mathbf{x}}_k+\alpha_k^2(t)\hat{\mathbf{s}}_{t,k}^p(\hat{\mathbf{X}}).
\end{equation}
Applying the same identity under the test endpoint distribution $q$ gives
\begin{equation}
    \beta_k(t)\hat{\boldsymbol{\mu}}^q_{t,k}(\mathbf{X})=\hat{\mathbf{x}}_k+\alpha_k^2(t)\hat{\mathbf{s}}_{t,k}^q(\hat{\mathbf{X}}).
\end{equation}
Thus, we have
\begin{equation}
    \hat{\boldsymbol{\mu}}^q_{t,k}(\mathbf{X})-\hat{\boldsymbol{\mu}}^p_{t,k}(\mathbf{X})=\frac{\alpha_k^2(t)}{\beta_k(t)}\left(\hat{\mathbf{s}}_{t,k}^q(\hat{\mathbf{X}})-\hat{\mathbf{s}}_{t,k}^p(\hat{\mathbf{X}})\right)
\end{equation}
Based on Lemma \ref{prop:target_velocity_scaling}, we know that
\begin{equation}
    \dot{\hat{\boldsymbol{\gamma}}}_k^\star(t)-\hat{\mathbf{v}}^p_{t,k}(\mathbf{X})=\rho_k(t)\left(\hat{\mathbf{x}}_{1,k}-\hat{\boldsymbol{\mu}}^p_{t,k}(\mathbf{X})\right).
\end{equation}
Thus, we have
\begin{equation}
\begin{aligned}
    \dot{\hat{\boldsymbol{\gamma}}}_k^\star(t)-\hat{\mathbf{v}}^p_{t,k}(\mathbf{X})&=\rho_k(t)\left(\hat{\mathbf{x}}_{1,k}-\hat{\boldsymbol{\mu}}^q_{t,k}(\mathbf{X})+\hat{\boldsymbol{\mu}}^q_{t,k}(\mathbf{X})-\hat{\boldsymbol{\mu}}^p_{t,k}(\mathbf{X})\right)\\
    &=\rho_k(t)\left(\hat{\mathbf{x}}_{1,k}-\hat{\boldsymbol{\mu}}^q_{t,k}(\mathbf{X})+\frac{\alpha_k^2(t)}{\beta_k(t)}\left(\hat{\mathbf{s}}_{t,k}^q(\hat{\mathbf{X}})-\hat{\mathbf{s}}_{t,k}^p(\hat{\mathbf{X}})\right)\right)
\end{aligned}
\end{equation}
Define $\mathbf{d}_{t,k}=\hat{\mathbf{x}}_{1,k}-\hat{\boldsymbol{\mu}}^q_{t,k}(\mathbf{X})$ and $\boldsymbol{\delta}_{t,k}=\hat{\mathbf{s}}_{t,k}^q(\hat{\mathbf{X}})-\hat{\mathbf{s}}_{t,k}^p(\hat{\mathbf{X}})$, then we obtain
\begin{equation}
    \dot{\hat{\boldsymbol{\gamma}}}_k^\star(t)-\hat{\mathbf{v}}^p_{t,k}(\mathbf{X})=\rho_k(t)\left(\mathbf{d}_{t,k}+\frac{\alpha_k^2(t)}{\beta_k(t)}\boldsymbol{\delta}_{t,k}\right).
\end{equation}
Thus,
\begin{equation}
\begin{aligned}
    \|\dot{\hat{\boldsymbol{\gamma}}}_k^\star(t)-\hat{\mathbf{v}}^p_{t,k}(\mathbf{X})\|_2^2&=\rho_k^2(t)\left\Vert\mathbf{d}_{t,k}+\frac{\alpha_k^2(t)}{\beta_k(t)}\boldsymbol{\delta}_{t,k}\right\Vert_2^2\\
    &=\rho_k^2(t)\left(\|\mathbf{d}_{t,k}\|_2^2+\frac{\alpha_k^4(t)}{\beta^2_k(t)}\|\boldsymbol{\delta}_{t,k}\|_2^2+2\frac{\alpha_k^2(t)}{\beta_k(t)}\mathbf{d}_{t,k}^\top\boldsymbol{\delta}_{t,k}\right).
\end{aligned}
\end{equation}
Since
\begin{equation}
    \hat{\boldsymbol{\mu}}^q_{t,k}(\mathbf{X})=\mathbb E_{p_0,q}[\hat{\boldsymbol{\gamma}}_k^\star(1)\mid \mathbf{X}_t=\mathbf{X}],
\end{equation}
we have
\begin{equation}
    \mathbb{E}_{p_0,q}\left[\mathbf{d}_{t,k}\mid\hat{\mathbf{X}}_t\right]=\mathbb{E}_{p_0,q}\left[\hat{\mathbf{x}}_{1,k}-\hat{\boldsymbol{\mu}}^q_{t,k}(\mathbf{X})\mid\hat{\mathbf{X}}_t\right]=\mathbf{0}.
\end{equation}
As $\boldsymbol{\delta}_{t,k}$ only depends on $\hat{\mathbf{X}}_t$, the cross term vanishes:
\begin{equation}
    \mathbb{E}_{p_0,q}\left[\langle\mathbf{d}_{t,k},\boldsymbol{\delta}_{t,k}\rangle\mid\hat{\mathbf{X}}_t\right]=\langle\mathbb{E}_{p_0,q}\left[\mathbf{d}_{t,k}\mid\hat{\mathbf{X}}_t\right],\boldsymbol{\delta}_{t,k}(\hat{\mathbf{X}}_t)\rangle=\mathbf{0}.
\end{equation}
Therefore
\begin{equation}
\begin{aligned}
    \mathbb{E}_{p_0,q}\left[\|\dot{\hat{\boldsymbol{\gamma}}}_k^\star(t)-\hat{\mathbf{v}}^p_{t,k}(\mathbf{X})\|_2^2\right]=\rho_k^2(t)\mathbb{E}_{p_0,q}\left[\|\mathbf{d}_{t,k}\|_2^2\right]+\rho_k^2(t)\frac{\alpha_k^4(t)}{\beta^2_k(t)}\mathbb{E}_{q_t}\left[\|\boldsymbol{\delta}_{t,k}\|_2^2\right].
\end{aligned}
\end{equation}
Since the anomaly score sums the contributions from these identically
distributed source samples in $\mathcal X_{\mathrm{ad}}$, linearity of expectation gives
\begin{equation}
\begin{aligned}
&\mathbb{E}_{\mathcal{X}_\mathrm{ad},q}\left[\xi(\mathcal{X}_\mathrm{ad},\mathcal{R}_\mathrm{ad},\mathbf{X}_1)\right] \\=& \sum_{k=1}^N\sum_{t\in\mathcal{R}_\mathrm{ad}}|\mathcal{X}_\mathrm{ad}|
 \left( \eta_k(t)\rho_k^2(t)\mathbb{E}_{p_0,q}\left[\|\mathbf{d}_{t,k}\|_2^2\right]+\eta_k(t)\rho_k^2(t)\frac{\alpha_k^4(t)}{\beta^2_k(t)}\mathbb{E}_{q_t}\left[\|\boldsymbol{\delta}_{t,k}\|_2^2\right]\right).    
\end{aligned}
\end{equation}
Observe that 
\begin{equation}
    \eta_k(t)\rho_k^2(t)=\frac{\sinh^2(\omega_kt)}{\omega^2_k}\left(\frac{\omega_k}{\sinh(\omega_k(1-t))}\right)^2=\frac{\sinh^2(\omega_kt)}{\sinh^2(\omega_k(1-t))}
\end{equation}
and
\begin{equation}
\begin{aligned}
    \eta_k(t)\rho_k^2(t)\frac{\alpha_k^4(t)}{\beta^2_k(t)}&=\frac{\sinh^2(\omega_kt)}{\sinh^2(\omega_k(1-t))}\left(\frac{\sinh(\omega_k(1-t))}{\sinh(\omega_k)}\right)^4\left(\frac{\sinh(\omega_k)}{\sinh(\omega_kt)}\right)^2\\&=\frac{\sinh^2(\omega_k(1-t))}{\sinh^2(\omega_k)},
\end{aligned}
\end{equation}
we have the following result:
\begin{equation}
\begin{aligned}
&\mathbb{E}_{\mathcal{X}_\mathrm{ad},q}\left[\xi(\mathcal{X}_\mathrm{ad},\mathcal{R}_\mathrm{ad},\mathbf{X}_1)\right] \\=&\sum_{k=1}^N\sum_{t\in\mathcal{R}_\mathrm{ad}}|\mathcal{X}_\mathrm{ad}|
 \left( \frac{\sinh^2(\omega_kt)}{\sinh^2(\omega_k(1-t))}\mathbb{E}_{p_0,q}\left[\|\mathbf{d}_{t,k}\|_2^2\right]+\frac{\sinh^2(\omega_k(1-t))}{\sinh^2(\omega_k)}\mathbb{E}_{q_t}\left[\|\boldsymbol{\delta}_{t,k}\|_2^2\right]\right),    
\end{aligned}
\end{equation}
through which the proof is completed.
\end{proof}

\subsection{Proof of Corollary \ref{prop:uncertainty_bound}}\label{app:uncertainty_bound}
\begin{proof}
For $t\in(0,1)$, define the rescaled spectral observation
\begin{equation}
\mathbf Z_{t,k}
=
\frac{\hat{\boldsymbol\gamma}_k^\star(t)}
     {\beta_k(t)}
=
\hat{\mathbf x}_{1,k}
+
\frac{\alpha_k(t)}{\beta_k(t)}
\hat{\mathbf x}_{0,k}.    
\end{equation}
Since
\begin{equation}
\hat{\boldsymbol\mu}_{t,k}^q(\mathbf X_t)
=
\mathbb E_{p_0,q}[
\hat{\mathbf x}_{1,k}\mid\mathbf X_t]    
\end{equation}
is the minimum mean-squared-error estimator of
$\hat{\mathbf x}_{1,k}$, using $\mathbf Z_{t,k}$ as a candidate
estimator gives
\begin{equation}
\begin{aligned}
\mathbb E_{p_0,q}[\|\mathbf d_{t,k}\|_2^2]
&\leq
\mathbb E_{p_0,q}[
\|\hat{\mathbf x}_{1,k}-\mathbf Z_{t,k}\|_2^2] \\
&=
\frac{\alpha_k^2(t)}{\beta_k^2(t)}
\mathbb E_{p_0}[\|\hat{\mathbf x}_{0,k}\|_2^2] \\
&=
R\frac{\alpha_k^2(t)}{\beta_k^2(t)}.
\end{aligned}
\label{eq:app_corollary_5}
\end{equation}
Because
\begin{equation}
\frac{\sinh^2(\omega_k t)}
     {\sinh^2(\omega_k(1-t))}
=
\frac{\beta_k^2(t)}{\alpha_k^2(t)},    
\end{equation}
multiplying both sides of \eqref{eq:app_corollary_5} by this factor gives
\begin{equation}
\frac{\sinh^2(\omega_k t)}
     {\sinh^2(\omega_k(1-t))}
\mathbb E_{p_0,q}[\|\mathbf d_{t,k}\|_2^2]
\leq R.    
\label{eq:app_corollary_6}
\end{equation}
For $\omega_k=0$, the ratio in \eqref{eq:app_corollary_6} is defined by its continuous limit $t^2/(1-t)^2$, and the same bound holds.
Finally, summing \eqref{eq:app_corollary_6} over all source samples, flow times, and graph-frequency
modes yields
\begin{equation}
\sum_{k=1}^N\sum_{t\in\mathcal{R}_\mathrm{ad}}|\mathcal{X}_\mathrm{ad}| \frac{\sinh^2(\omega_kt)}{\sinh^2(\omega_k(1-t))}\mathbb{E}_{p_0,q}\left[\|\hat{\mathbf{x}}_{1,k}-\hat{\boldsymbol{\mu}}^q_{t,k}(\mathbf{X})\|_2^2\right]
\leq
|\mathcal X_{\mathrm{ad}}|\,
|\mathcal R_{\mathrm{ad}}|\,
NR,
\end{equation}
which completes the proof.
\end{proof}

\section{Experimental Details and Additional Results}\label{app:experiments}

\subsection{Datasets}\label{app:datasets}
We evaluate GRASP on four MTS anomaly detection datasets from mTSBench \citep{zhou2026mtsbench}, covering spacecraft telemetry, server monitoring, network intrusion detection, and space-weather analysis. 
Each dataset contains separate training and test sequences, together with pointwise anomaly labels for evaluation. We use the data preprocessing and benchmark splits provided by mTSBench. We briefly introduce the four datasets below:
\begin{itemize}
    \item SMAP \citep{hundman2018detecting}: The Soil Moisture Active Passive dataset contains telemetry collected from NASA’s SMAP spacecraft. The benchmark subset comprises 51 multivariate sequences, each containing 26 variables. This dataset evaluates the ability to detect anomalous behavior in spacecraft telemetry.
\item SMD \citep{su2019robust}: The Server Machine Dataset contains monitoring measurements collected from server machines. We use 18 sequences with 39 variables each, where the labeled anomalies represent deviations from normal server operation.
\item CICIDS \citep{sharafaldin2018toward}: CICIDS2017 is a network intrusion detection dataset containing both benign traffic and multiple types of cyberattacks. The benchmark subset comprises six sequences, each represented by 73 network-flow features.
\item SWAN \citep{angryk2020swan}: The Space Weather Analytics for Solar Flares dataset contains multivariate time series of physical properties associated with solar active regions. We use the 39-variable sequence included in the mTSBench anomaly detection benchmark.
\end{itemize}

\subsection{Baselines}
We compare GRASP with 13 representative baselines spanning statistical outlier detection, forecasting- and reconstruction-based detection, graph-based modeling, Transformers, and frequency-domain methods. We briefly introduce these baselines below:
\begin{itemize}
\item \textbf{GDN} \citep{deng2021graph} learns inter-sensor dependencies using node embeddings and graph attention, and detects anomalies using forecasting errors.

\item \textbf{GCAD} \citep{liu2025gcad} infers dynamic Granger-causal graphs from the gradients of a forecasting model and identifies anomalies through deviations in the learned causal patterns.

\item \textbf{COPOD} \citep{li2020copod} estimates empirical copula-based tail probabilities and assigns larger anomaly scores to statistically extreme observations.

\item \textbf{HBOS} \citep{goldstein2012histogram} constructs a histogram for each variable and combines the resulting density estimates under a feature-independence assumption.

\item \textbf{TimesNet} \citep{wu2023timesnet} transforms one-dimensional time series into period-dependent two-dimensional representations to capture intra- and inter-period variations.

\item \textbf{CNN} \citep{munir2018deepant} uses temporal convolutions to forecast future observations and detects anomalies through prediction errors.

\item \textbf{USAD} \citep{audibert2020usad} employs adversarially trained Autoencoders to learn normal temporal patterns and scores anomalies using reconstruction errors.

\item \textbf{TranAD} \citep{TranAD} combines Transformer-based sequence modeling, self-conditioning, and adversarial training for MTS anomaly detection.

\item \textbf{OmniAnomaly} \citep{su2019robust} combines recurrent temporal modeling with a variational Autoencoder to capture stochastic temporal dependencies and detects anomalies using reconstruction probabilities.

\item \textbf{Autoencoder} \citep{sakurada2014anomaly} learns a compressed representation of normal observations and uses reconstruction errors as anomaly scores.

\item \textbf{A-Transformer} \citep{xu2022anomaly} models prior and learned temporal associations through anomaly attention and detects anomalies using the discrepancy between these associations.

\item \textbf{FITS} \citep{xu2024fits} models time series through learnable interpolation of low-frequency components in the complex frequency domain.

\item \textbf{CATCH} \citep{wu2025catch} divides frequency-domain representations into patches and captures spectral patterns and inter-channel dependencies through masked channel fusion.
\end{itemize}
\begin{table*}[t]
\centering
\caption{
Anomaly detection performance comparison on SMAP and SMD.
}\vspace{2pt}
\label{tab:main_results_smap_smd}
\resizebox{\textwidth}{!}{
\begin{tabular}{l|ccc|ccc}
\toprule
\multirow{2}{*}{Model}
& \multicolumn{3}{c|}{SMAP}
& \multicolumn{3}{c}{SMD} \\
& PRC & ROC & Best-F1
& PRC & ROC & Best-F1 \\
\midrule

A-Transformer
& 0.1340 $\pm$ 0.0072
& 0.5065 $\pm$ 0.0071
& 0.2060 $\pm$ 0.0110
& 0.0720 $\pm$ 0.0046
& 0.5013 $\pm$ 0.0049
& 0.1260 $\pm$ 0.0071 \\

FITS
& 0.1455 $\pm$ 0.0047
& 0.5147 $\pm$ 0.0068
& 0.2798 $\pm$ 0.0063
& 0.2863 $\pm$ 0.0026
& 0.8299 $\pm$ 0.0022
& 0.4025 $\pm$ 0.0026 \\

TimesNet
& 0.1792 $\pm$ 0.0071
& 0.5336 $\pm$ 0.0058
& 0.3161 $\pm$ 0.0089
& 0.2839 $\pm$ 0.0053
& 0.8087 $\pm$ 0.0021
& 0.3713 $\pm$ 0.0049 \\

COPOD
& 0.1927 $\pm$ 0.0000
& 0.5918 $\pm$ 0.0000
& 0.2841 $\pm$ 0.0000
& 0.2139 $\pm$ 0.0000
& 0.7193 $\pm$ 0.0000
& 0.2951 $\pm$ 0.0000 \\

HBOS
& 0.1897 $\pm$ 0.0000
& 0.5650 $\pm$ 0.0000
& 0.2664 $\pm$ 0.0000
& 0.2678 $\pm$ 0.0000
& 0.7378 $\pm$ 0.0000
& 0.3556 $\pm$ 0.0000 \\

TranAD
& 0.2122 $\pm$ 0.0042
& 0.6104 $\pm$ 0.0056
& 0.3264 $\pm$ 0.0066
& 0.3304 $\pm$ 0.0022
& 0.7691 $\pm$ 0.0022
& 0.4181 $\pm$ 0.0018 \\

CNN
& 0.2136 $\pm$ 0.0029
& \underline{0.6577 $\pm$ 0.0088}
& 0.3471 $\pm$ 0.0055
& 0.3730 $\pm$ 0.0049
& 0.7921 $\pm$ 0.0042
& 0.4313 $\pm$ 0.0046 \\

Autoencoder
& 0.2277 $\pm$ 0.0006
& 0.5859 $\pm$ 0.0015
& 0.2991 $\pm$ 0.0006
& 0.3613 $\pm$ 0.0050
& 0.7297 $\pm$ 0.0034
& 0.4214 $\pm$ 0.0034 \\

USAD
& 0.2263 $\pm$ 0.0049
& 0.5850 $\pm$ 0.0089
& 0.3850 $\pm$ 0.0075
& 0.4105 $\pm$ 0.0040
& 0.8726 $\pm$ 0.0019
& 0.4810 $\pm$ 0.0021 \\

OmniAnomaly
& 0.2420 $\pm$ 0.0003
& 0.5950 $\pm$ 0.0002
& \underline{0.4065 $\pm$ 0.0003}
& 0.4284 $\pm$ 0.0001
& \underline{0.8822 $\pm$ 0.0001}
& \underline{0.5061 $\pm$ 0.0001} \\

CATCH
& \underline{0.2450 $\pm$ 0.0028}
& 0.6078 $\pm$ 0.0046
& 0.3664 $\pm$ 0.0041
& \underline{0.4763 $\pm$ 0.0011}
& \textbf{0.8968 $\pm$ 0.0004}
& 0.5057 $\pm$ 0.0012 \\

GDN
& 0.2197 $\pm$ 0.0115
& 0.6099 $\pm$ 0.0070
& 0.3394 $\pm$ 0.0087
& 0.4006 $\pm$ 0.0157
& 0.8014 $\pm$ 0.0104
& 0.4640 $\pm$ 0.0140 \\

GCAD
& 0.2144 $\pm$ 0.0135
& 0.6348 $\pm$ 0.0119
& 0.3471 $\pm$ 0.0104
& 0.3569 $\pm$ 0.0089
& 0.8416 $\pm$ 0.0074
& 0.4535 $\pm$ 0.0109 \\

\midrule
\rowcolor{gray!12}
\textbf{GRASP}
& \textbf{0.3149 $\pm$ 0.0087}
& \textbf{0.7015 $\pm$ 0.0130}
& \textbf{0.4178 $\pm$ 0.0064}
& \textbf{0.4845 $\pm$ 0.0043}
& 0.8567 $\pm$ 0.0038
& \textbf{0.5087 $\pm$ 0.0034} \\

\bottomrule
\end{tabular}
}
\end{table*}
\subsection{Implementation Details}\label{app:implementation_details}
We implement GRASP in PyTorch and evaluate it on SMAP, SMD, CICIDS, and SWAN. For each time series, the normal training data are chronologically divided into 80\% for training and 20\% for validation. We use a window length of 50 and set the graph Dirichlet energy weight parameter to $\tau=2$ for all datasets. 
For graph construction, we binarize the edge-weight matrix using a fixed threshold of 0.5 across all datasets.

The velocity predictor used in GRASP consists of two time-conditioned TSMixer blocks with a hidden dimension of 128 and a dropout rate of 0.1. Flow time is encoded with a fixed 64-dimensional sinusoidal embedding. 
The model is trained with a batch size of 256 and a learning rate of $10^{-3}$. 
Training is limited to 1500 epochs, and validation is performed every 50 epochs. 
No anomaly labels are used for training, validation, graph construction, or anomaly score calibration. 

During inference, anomaly scores are computed from the weighted velocity discrepancies aggregated across graph-frequency modes, source samples, and flow times. 
We compute the average of five source samples, and each prior sample is evaluated at ten flow times sampled equally between 0 and 1.
All experiments are conducted on a single 80GB NVIDIA A100 GPU. We repeat each experiment using ten random seeds.

\subsection{Additional MTS Anomaly Detection Results}
In Section \ref{sec:anomaly_detection_performance}, we have presented the mean MTS anomaly detection performance over 10 random seeds in Table \ref{tab:main_results}.
To further assess the stability of the compared methods, we further report the results with standard deviation on the SMAP and SMD datasets in Table \ref{tab:main_results_smap_smd}, and on the CICIDS and SWAN datasets in Table \ref{tab:main_results_cicids_swan}.

GRASP achieves the highest PRC and Best-F1 on all four datasets, as
well as the highest ROC on three of them. On SMD, GRASP obtains the
best PRC and Best-F1, while CATCH achieves the highest ROC. The
relatively small standard deviations of GRASP indicate that its
performance is stable across random seeds.

\begin{table*}[t]
\centering
\caption{
Anomaly detection performance comparison on CICIDS and SWAN.
}\vspace{2pt}
\label{tab:main_results_cicids_swan}
\resizebox{\textwidth}{!}{
\begin{tabular}{l|ccc|ccc}
\toprule
\multirow{2}{*}{Model}
& \multicolumn{3}{c|}{CICIDS}
& \multicolumn{3}{c}{SWAN} \\
& PRC & ROC & Best-F1
& PRC & ROC & Best-F1 \\
\midrule

A-Transformer
& 0.2396 $\pm$ 0.0007
& 0.4972 $\pm$ 0.0025
& 0.3060 $\pm$ 0.0000
& 0.2817 $\pm$ 0.0018
& 0.5021 $\pm$ 0.0014
& 0.4365 $\pm$ 0.0000 \\

FITS
& 0.1849 $\pm$ 0.0013
& 0.4106 $\pm$ 0.0039
& 0.3126 $\pm$ 0.0013
& 0.3189 $\pm$ 0.0120
& 0.4607 $\pm$ 0.0120
& 0.4365 $\pm$ 0.0000 \\

TimesNet
& 0.1905 $\pm$ 0.0021
& 0.4657 $\pm$ 0.0052
& 0.3121 $\pm$ 0.0015
& 0.3068 $\pm$ 0.0080
& 0.4505 $\pm$ 0.0137
& 0.4365 $\pm$ 0.0000 \\

COPOD
& 0.2214 $\pm$ 0.0000
& 0.5566 $\pm$ 0.0000
& 0.3733 $\pm$ 0.0000
& 0.2792 $\pm$ 0.0000
& 0.5000 $\pm$ 0.0000
& 0.4365 $\pm$ 0.0000 \\

HBOS
& 0.2603 $\pm$ 0.0000
& 0.5586 $\pm$ 0.0000
& 0.3852 $\pm$ 0.0000
& 0.2792 $\pm$ 0.0000
& 0.5000 $\pm$ 0.0000
& 0.4365 $\pm$ 0.0000 \\

TranAD
& 0.2043 $\pm$ 0.0021
& 0.4704 $\pm$ 0.0085
& 0.3416 $\pm$ 0.0007
& 0.3360 $\pm$ 0.0020
& 0.4911 $\pm$ 0.0039
& 0.4365 $\pm$ 0.0001 \\

CNN
& 0.2371 $\pm$ 0.0019
& 0.5543 $\pm$ 0.0056
& 0.3991 $\pm$ 0.0009
& 0.3760 $\pm$ 0.0056
& 0.4966 $\pm$ 0.0086
& 0.4365 $\pm$ 0.0000 \\

Autoencoder
& 0.3147 $\pm$ 0.0089
& \underline{0.7276 $\pm$ 0.0408}
& 0.3938 $\pm$ 0.0106
& 0.3294 $\pm$ 0.0025
& 0.4944 $\pm$ 0.0091
& 0.4371 $\pm$ 0.0006 \\

USAD
& 0.2094 $\pm$ 0.0007
& 0.4619 $\pm$ 0.0019
& 0.3525 $\pm$ 0.0003
& 0.4405 $\pm$ 0.0088
& 0.5585 $\pm$ 0.0095
& 0.4365 $\pm$ 0.0000 \\

OmniAnomaly
& 0.2138 $\pm$ 0.0000
& 0.4501 $\pm$ 0.0000
& 0.3548 $\pm$ 0.0000
& 0.4847 $\pm$ 0.0002
& 0.6130 $\pm$ 0.0004
& 0.4365 $\pm$ 0.0000 \\

CATCH
& 0.2866 $\pm$ 0.0007
& 0.6162 $\pm$ 0.0006
& \underline{0.4023 $\pm$ 0.0008}
& 0.4638 $\pm$ 0.0017
& 0.6218 $\pm$ 0.0018
& 0.4533 $\pm$ 0.0009 \\

GDN
& 0.2926 $\pm$ 0.0188
& 0.7017 $\pm$ 0.0239
& 0.3805 $\pm$ 0.0148
& 0.6423 $\pm$ 0.0231
& \underline{0.8356 $\pm$ 0.0091}
& \underline{0.6432 $\pm$ 0.0155} \\

GCAD
& \underline{0.3636 $\pm$ 0.0194}
& 0.7191 $\pm$ 0.0151
& 0.3844 $\pm$ 0.0055
& \underline{0.6445 $\pm$ 0.0146}
& 0.8140 $\pm$ 0.0070
& 0.6231 $\pm$ 0.0089 \\

\midrule
\rowcolor{gray!12}
\textbf{GRASP}
& \textbf{0.3996 $\pm$ 0.0083}
& \textbf{0.8108 $\pm$ 0.0072}
& \textbf{0.4615 $\pm$ 0.0037}
& \textbf{0.7556 $\pm$ 0.0116}
& \textbf{0.8758 $\pm$ 0.0101}
& \textbf{0.6886 $\pm$ 0.0145} \\

\bottomrule
\end{tabular}
}
\end{table*}
\begin{table*}[t]
\centering
\caption{
Ablation results on CICIDS and SWAN.
}\vspace{2pt}
\label{tab:path_ablation}
\resizebox{0.7\textwidth}{!}{
\begin{tabular}{l|ccc|ccc}
\toprule
\multirow{2}{*}{Variant}
& \multicolumn{3}{c|}{CICIDS}
& \multicolumn{3}{c}{SWAN} \\
& PRC & ROC & Best-F1
& PRC & ROC & Best-F1 \\
\midrule
\rowcolor{gray!12}
GRASP
& \textbf{0.3996} & \textbf{0.8108} & \textbf{0.4615}
& \textbf{0.7556} & \textbf{0.8758} & \textbf{0.6886}\\
GRASP-mean
& \underline{0.3764} & \underline{0.7593} & \underline{0.4499}
& \underline{0.7552} & \underline{0.8753} & \underline{0.6881} \\
GRASP-ER
& 0.3734 & 0.7546 & 0.4497
& 0.7534 & 0.8723 & 0.6857 \\

GRASP-lin
& 0.3740 & 0.7543 & 0.4493
& 0.7486 & 0.8714 & 0.6841 \\

\bottomrule
\end{tabular}
}
\end{table*}
\begin{table*}[t]
\centering
\caption{
Ablation study of velocity-field architectures.
}\vspace{2pt}
\label{tab:velocity_field_ablation}
\resizebox{\textwidth}{!}{
\begin{tabular}{l|ccc|ccc|ccc|ccc}
\toprule
\multirow{2}{*}{Variant}
& \multicolumn{3}{c|}{SMAP}
& \multicolumn{3}{c|}{SMD}
& \multicolumn{3}{c|}{CICIDS}
& \multicolumn{3}{c}{SWAN} \\
& PRC & ROC & Best-F1
& PRC & ROC & Best-F1
& PRC & ROC & Best-F1
& PRC & ROC & Best-F1 \\
\midrule
\rowcolor{gray!12}
TSMixer
& \textbf{0.3149} & \textbf{0.7015} & \textbf{0.4178}
& \textbf{0.4845} & \textbf{0.8567} & \textbf{0.5087}
& \textbf{0.3996} & \textbf{0.8108} & \textbf{0.4615}
& \textbf{0.7556} & \textbf{0.8758} & \textbf{0.6886} \\

GNN
& \underline{0.2773} & \underline{0.6758} & \underline{0.3881}
& \underline{0.4659} & \underline{0.8542} & \underline{0.4926}
& \underline{0.3551} & \underline{0.7232} & \underline{0.4459}
& 0.5306 & 0.6570 & \underline{0.5188} \\

Transformer
& 0.2639 & 0.6347 & 0.3795
& 0.4475 & 0.8272 & 0.4674
& 0.3283 & 0.6407 & 0.4331
& \underline{0.5594} & \underline{0.6758} & 0.5122 \\

\bottomrule
\end{tabular}
}
\end{table*}
\subsection{Additional Ablation Results}\label{app:ablation}
Table~\ref{tab:path_ablation} extends the ablation study in Table \ref{tab:ablation} to CICIDS and SWAN.
The GRASP model consistently achieves the best performance across all metrics.
The improvements are particularly clear on CICIDS, whereas the gains on SWAN are smaller but remain consistent.
Overall, replacing the graph-spectral path with a linear path, removing the graph-spectral weighting scheme, or using a randomly generated graph leads to performance degradation. 
These results validate the individual contribution of the three design components.

\subsection{Comparison of Velocity-Field Architectures}\label{app:ablation_velocity_model}
In this section, we examine the effect of the velocity-field architecture by replacing the default TSMixer backbone of GRASP with two time-conditioned alternatives based on a GNN and a Transformer.
The GNN treats sensor channels as nodes and applies two residual GCN blocks for cross-channel propagation \citep{kipf2017semisupervised}. 
The Transformer treats sensor channels as tokens and employs two encoder layers with four-head self-attention \citep{vaswani2017attention}. 
All three architectures use a hidden dimension of 128 and are trained and evaluated under the same experimental protocol.

As shown in Table~\ref{tab:velocity_field_ablation}, TSMixer achieves the highest PRC, ROC, and Best-F1 on all four datasets. 
The advantage is particularly pronounced on SWAN, where its PRC reaches 0.7552, compared with 0.5306 for the GNN and 0.5594 for the Transformer.
These results support the use of the simple yet effective TSMixer model in GRASP.
Under the considered setting, the tested GNN and Transformer alternatives provide no performance improvement.

\subsection{Additional Flow-Time and Graph-Frequency Analysis}\label{app:attribution}
Figure \ref{fig:attribution_app} extends the anomaly attribution analysis in Section \ref{sec:signal_attribution_analysis} to SMD, CICIDS, and SWAN datasets.
We decompose the squared velocity residuals across flow time and graph-frequency quantiles. 
The graph frequencies are ordered from low to high Laplacian eigenvalues, while larger flow times correspond to states closer to the data endpoint. 
The top row presents the unweighted residual contributions, and the bottom row presents the contributions
after applying the weighting schedule.

On SMD and CICIDS, the anomaly evidence is primarily concentrated in intermediate graph-frequency quantiles at later flow times. 
Applying the score weights further shifts the attribution toward the data endpoint while preserving the dominant frequency regions. 
In contrast, SWAN exhibits a more diffuse attribution pattern across graph frequencies. 
These results suggest that the flow time toward the endpoint provides more informative anomaly signals for all datasets, while the informative graph-frequency components vary across datasets rather than being universally concentrated in the highest-frequency modes.
\begin{figure}[t]
    \centering
    \includegraphics[width=1.0\columnwidth]{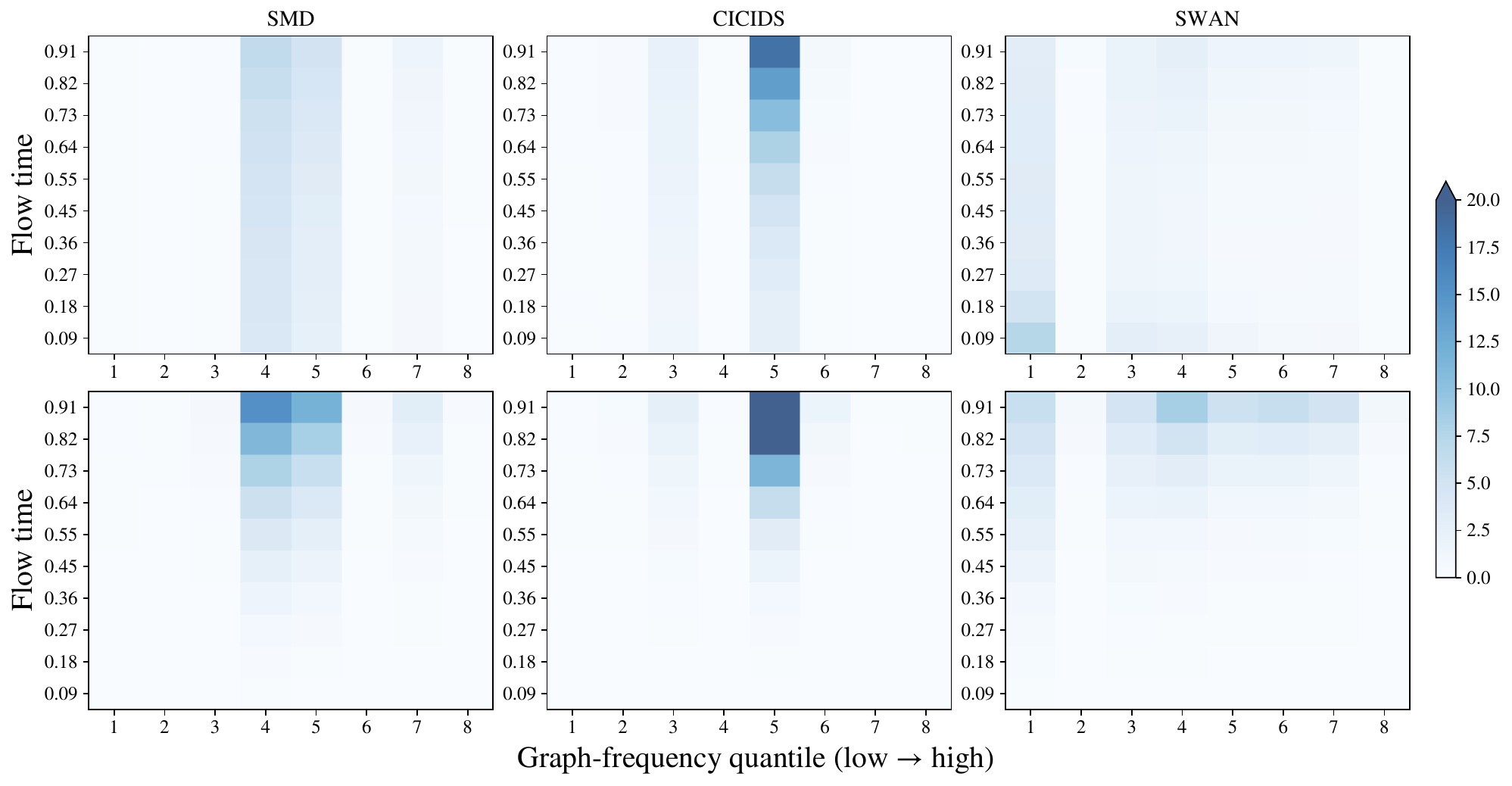}
    \caption{Anomaly score contributions across flow times and graph-frequency quantiles on SMD, CICIDS, and SWAN. Up: unweighted anomaly signal. Below: weighted anomaly signal.}\label{fig:attribution_app}
\end{figure}

\subsection{Sensitivity to Flow-Time Evaluations and Source Samples} \label{app:sensitivity_number_evaluations}
In this section, we examine the sensitivity of GRASP to two inference-time hyperparameters: the number of flow-time evaluation points $|\mathcal{R}_{\mathrm{ad}}|$ and the number of source samples $|\mathcal{X}_{\mathrm{ad}}|$.
We vary $|\mathcal{R}_{\mathrm{ad}}|\in\{1,2,5,10,20,50\}$ while fixing $|\mathcal{X}_{\mathrm{ad}}|=5$, and vary $|\mathcal{X}_{\mathrm{ad}}|\in\{1,2,5,10,20\}$ while fixing $|\mathcal{R}_{\mathrm{ad}}|=10$. 
Figure \ref{fig:sensitivity_T} and Figure \ref{fig:sensitivity_K} report the mean and standard deviation over ten random seeds.

As shown in Figure \ref{fig:sensitivity_T}, increasing the number of flow-time evaluation points generally improves performance on SMAP, SMD, and CICIDS, especially when $|\mathcal{R}_{\mathrm{ad}}|<10$. 
Performance on SWAN remains stable across different values of $|\mathcal{R}_{\mathrm{ad}}|$. 
Figure~\ref{fig:sensitivity_K} shows that GRASP is relatively insensitive to the number of source samples. Increasing $|\mathcal{X}_{\mathrm{ad}}|$ yields modest improvements on SMAP and negligible changes on the other datasets. 
These results support the default choices of $|\mathcal{R}_{\mathrm{ad}}|=10$ and $|\mathcal{X}_{\mathrm{ad}}|=5$ as practical trade-offs between detection performance and inference cost.

\begin{figure}[t]
    \centering
    \includegraphics[width=1.0\columnwidth]{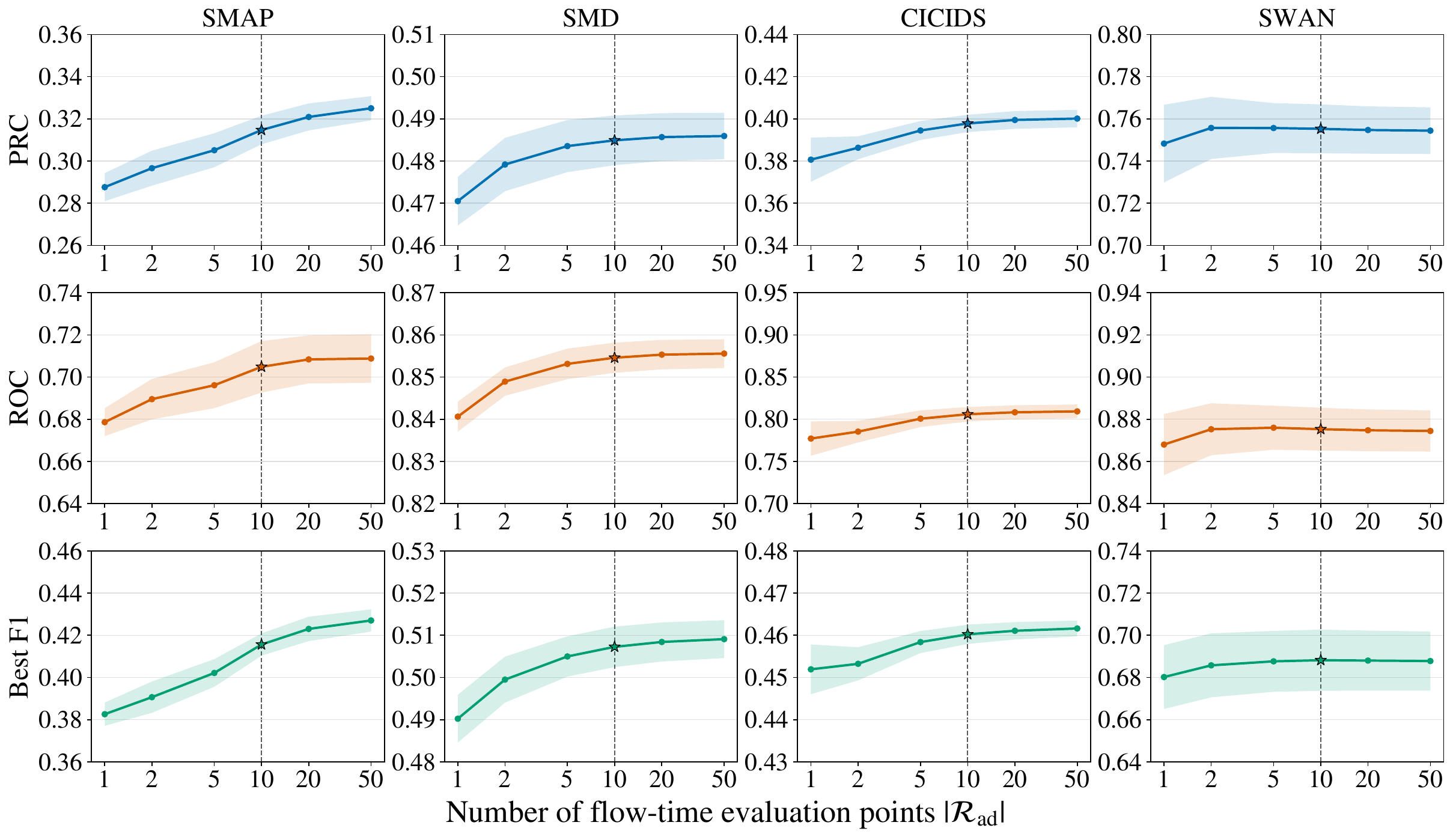}
    \caption{Sensitivity to the number of flow-time evaluation points
$|\mathcal{R}_{\mathrm{ad}}|$ with
$|\mathcal{X}_{\mathrm{ad}}|=5$.}
    \label{fig:sensitivity_T}
\end{figure}

\begin{figure}[t]
    \centering
    \includegraphics[width=1.0\columnwidth]{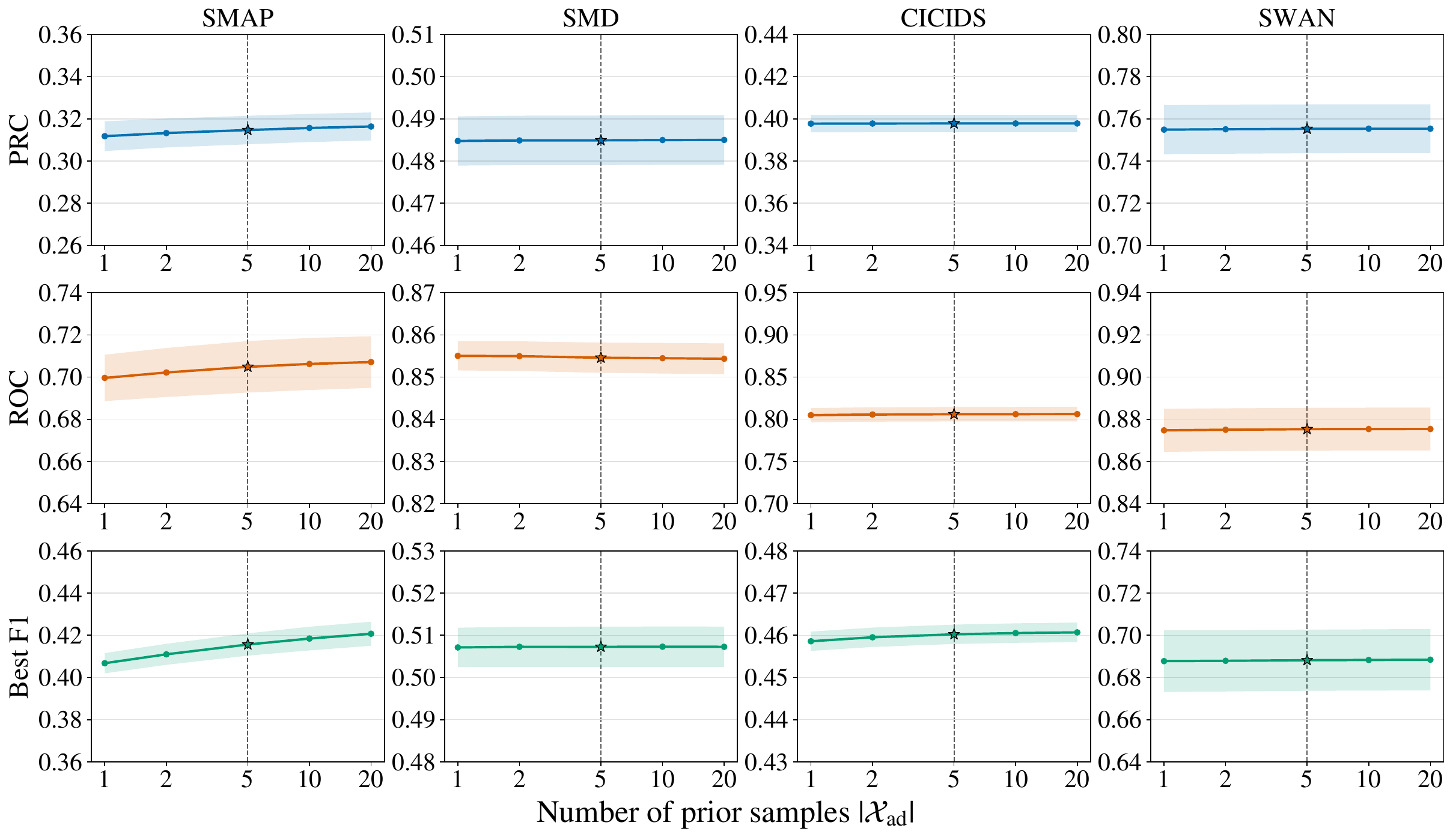}
    \caption{Sensitivity to the number of source samples
$|\mathcal{X}_{\mathrm{ad}}|$ with
$|\mathcal{R}_{\mathrm{ad}}|=10$.}
    \label{fig:sensitivity_K}
\end{figure}

\subsection{Sensitivity to the Graph Dirichlet-Energy Weight}
\label{app:sensitivity_tau}
The graph Dirichlet-energy weight $\tau$ controls the strength of the graph-structural regularization in the graph-spectral probability path, with $\tau=0$ recovering the standard linear path. 
We evaluate $\tau \in \{0, 0.25, 0.5, 1, 2, 4, 8\}$ on all four datasets, retraining the model for each value while keeping all other hyperparameters fixed.

Figure \ref{fig:sensitivity_tau} shows that the preferred value varies across datasets. 
SMD generally favors smaller values around $\tau=0.5$, whereas SMAP and CICIDS attain their best PRC and Best-F1 near $\tau=4$. 
SWAN performs best over an intermediate range around $\tau=1$--$2$. 
Despite these dataset-specific optima, performance remains relatively stable for $\tau \in [0.5,4]$.
In contrast, $\tau=8$ leads to a clear performance degradation, suggesting that excessively strong graph regularization could weaken the anomaly signal.

\begin{figure}[t]
    \centering
    \includegraphics[width=1.0\columnwidth]{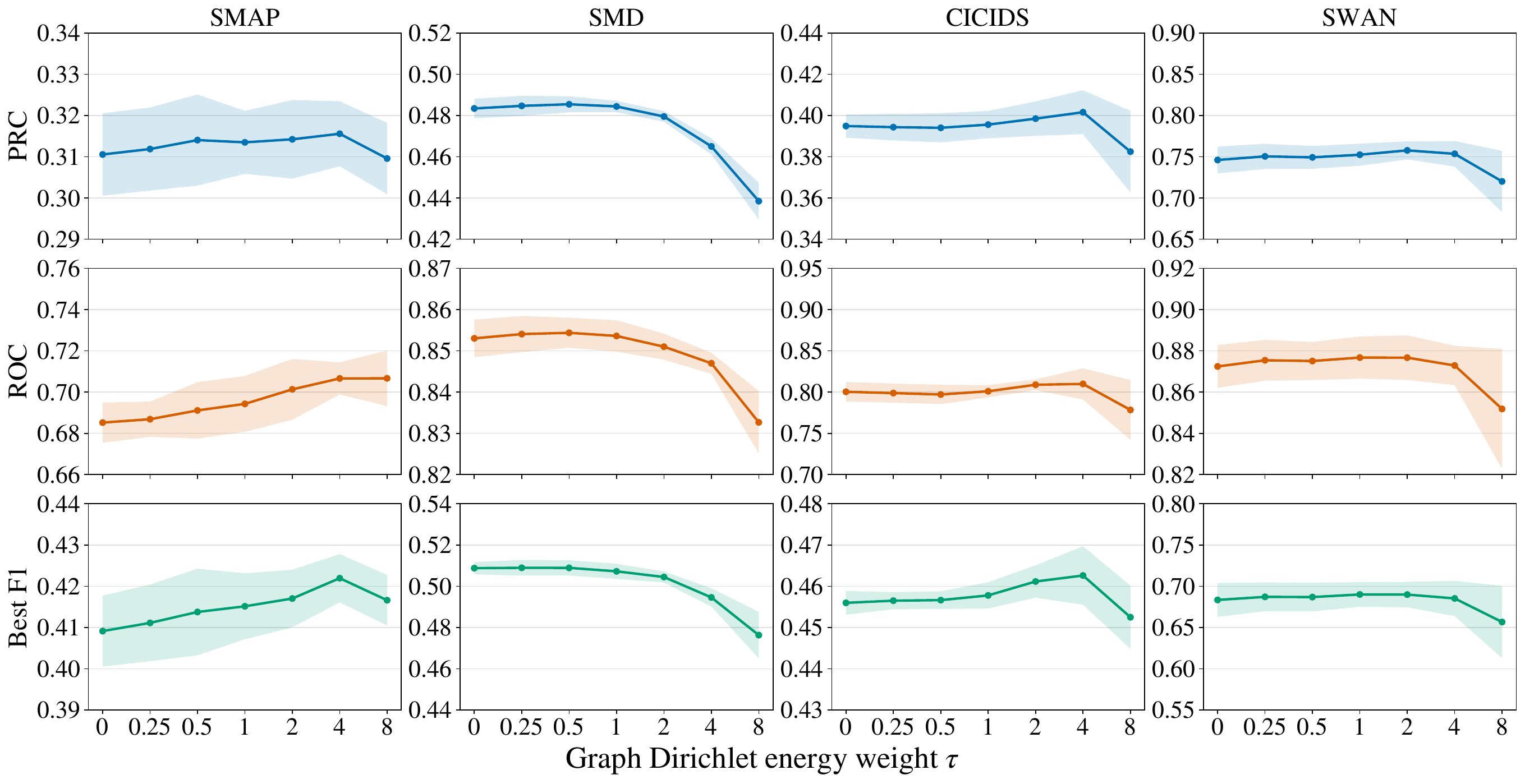}
    \caption{Sensitivity to the graph Dirichlet-energy weight $\tau$.}
    \label{fig:sensitivity_tau}
\end{figure}
\begin{figure}[t]
    \centering
    \includegraphics[width=1.0\columnwidth]{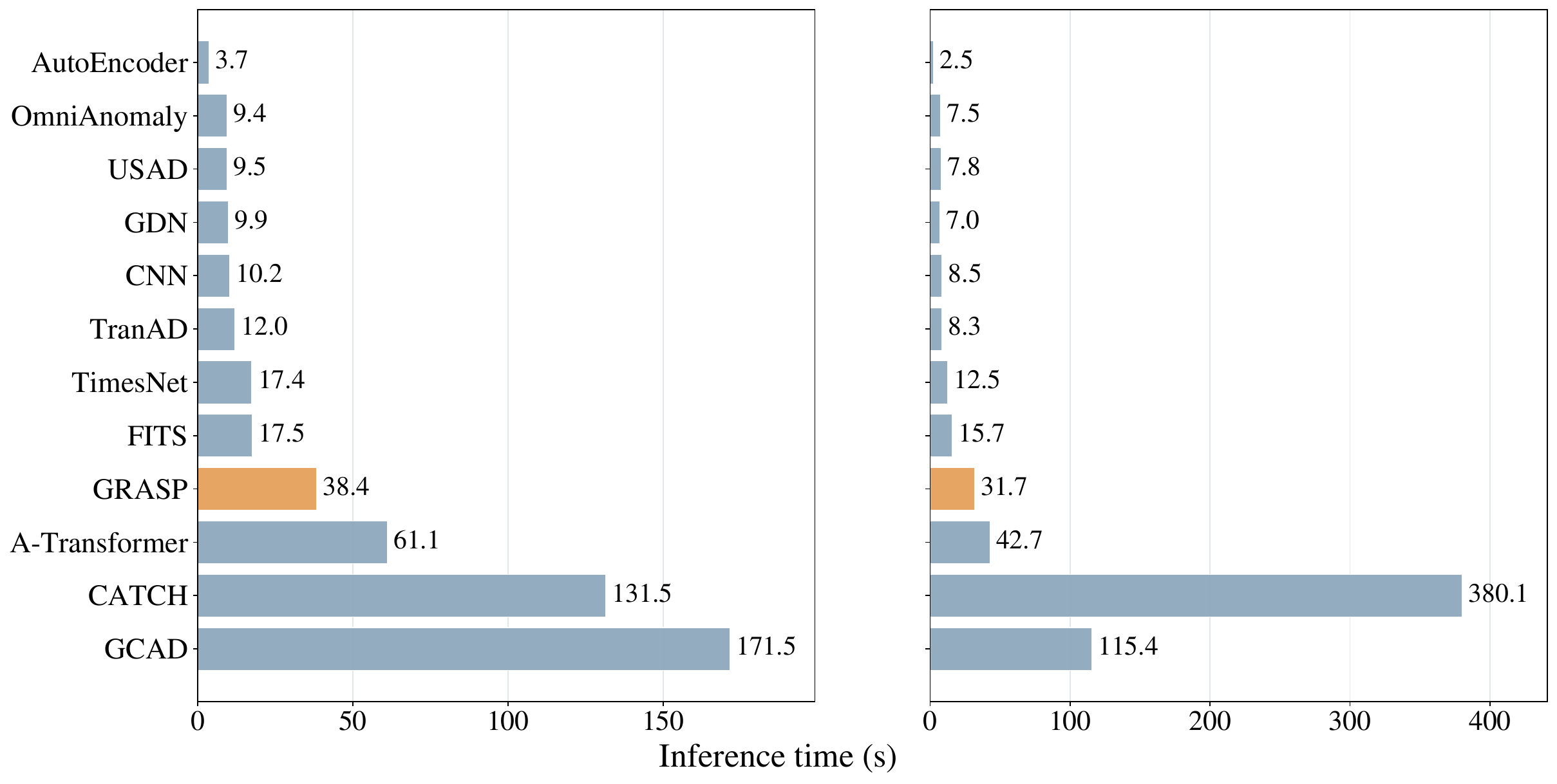}
    \caption{Inference time comparison on SMAP (left) and SMD (right).}
    \label{fig:inference_time}
\end{figure}
\subsection{Inference Efficiency}\label{app:inference_time}
In this section, we evaluate the inference efficiency of GRASP.
Specifically, we compare the inference time of anomaly detectors on the test set of SMAP and SMD.
For GRASP, we use the default setting with $|\mathcal{R}_{\mathrm{ad}}|=10$ flow-time evaluation points and $|\mathcal{X}_{\mathrm{ad}}|=5$ source samples.
Note that the eigenvalue decomposition is not included in the inference time computation as it is only performed once for each dataset before training and reused during inference.

As shown in Figure \ref{fig:inference_time}, GRASP requires 38.4 seconds on SMAP and 31.7 seconds on SMD to score the full test sets.
GRASP demonstrates faster inference compared with A-Transformer, CATCH, and GCAD on both datasets, but slower than the remaining baselines because its anomaly score aggregates multiple flow-time evaluations and source samples.
For application scenarios where faster inference is preferred, GRASP also provides a controllable trade-off between inference efficiency and detection performance.
Specifically, as discussed in Section \ref{app:sensitivity_number_evaluations}, GRASP can be further accelerated by reducing the flow-time evaluation points $|\mathcal{R}_{\mathrm{ad}}|$ and source samples $|\mathcal{X}_{\mathrm{ad}}|$ with a bit of compromised performance.


\end{document}